\documentclass{article}
\usepackage{graphicx}
\usepackage{subfigure} 

\usepackage{tikz}
\usepackage{pgfplots}
\usetikzlibrary{decorations.pathreplacing}
\usepgfplotslibrary{groupplots}
\usetikzlibrary{plotmarks}
\usepackage{natbib}

\usepackage{algorithm}
\usepackage{algorithmic}

\usepackage{hyperref}

\usepackage{icml2012} 
 \icmltitlerunning{Efficiently Sampling Multiplicative Attribute Graphs}

\input{Definitions}

\begin{document}

\twocolumn[ \icmltitle{Efficiently Sampling Multiplicative Attribute
  Graphs\\ Using a Ball-Dropping Process}

% It is OKAY to include author information, even for blind
% submissions: the style file will automatically remove it for you
% unless you've provided the [accepted] option to the icml2012
% package.
\icmlauthor{Your Name}{email@yourdomain.edu} \icmladdress{Your
  Fantastic Institute, 314159 Pi St., Palo Alto, CA 94306 USA}
\icmlauthor{Your CoAuthor's Name}{email@coauthordomain.edu}
\icmladdress{Their Fantastic Institute, 27182 Exp St., Toronto, ON M6H
  2T1 CANADA}

% You may provide any keywords that you find helpful for describing
% your paper; these are used to populate the "keywords" metadata in
% the PDF but will not be shown in the document
\icmlkeywords{graph model, social networks, machine learning,
  sampling, accept-reject, multiplicative attribute graphs, stochastic
  kronecker graphs}

\vskip 0.3in 
]

\begin{abstract}
  We introduce a novel and efficient sampling algorithm for the
  Multiplicative Attribute Graph Model (MAGM - \citet{KimLes10}).  Our
  algorithm is \emph{strictly} more efficient than the algorithm
  proposed by \citet{YunVis12}, in the sense that our method extends the
  \emph{best} time complexity guarantee of their algorithm to a larger
  fraction of parameter space. Both in theory and in empirical
  evaluation on sparse graphs, our new algorithm outperforms the
  previous one.

  To design our algorithm, we first define a stochastic
  \emph{ball-dropping process} (BDP). Although a special case of this
  process was introduced as an efficient approximate sampling algorithm
  for the Kronecker Product Graph Model (KPGM -
  \citet{LesChaKleFaletal10}), neither \emph{why} such an apprximation
  works nor \emph{what} is the actual distribution this process is
  sampling from has been addressed so far to the best of our knowledge.

  Our rigorous treatment of the BDP enables us to clarify the rational
  behind a BDP approximation of KPGM, and design an efficient sampling
  algorithm for the MAGM.
\end{abstract}

\section{Introduction}
\label{sec:intro}
In this paper we are concerned with statistical models on graphs.  The
scalability of the model's inference and sampling algorithm is becoming
a critical issue especially for sparse graphs, as more and more graph
data is becoming available. For instance, one can easily crawl a graph
with millions of nodes in few days from Twitter.

In this regard, the Kronecker Product Graph Model (KPGM) of
\citet{LesChaKleFaletal10} is particularly attractive. In contrast to
traditional models such as Exponential Random Graph Model (ERGM) of
\citet{RobPatKalLus07} or Latent Factor Model of \citet{Hoff09} which
cannot scale beyond graphs with thousands of nodes, both inference in
and sampling from a KPGM scale to graphs with millions of nodes.

However, the model has recently been criticized to be not very
realistic, both in theory \citep{SesPinKol11} and in practice
\citep{MorNev09}.  This is actually not very surprising, as the KPGM is
clearly under-parametrized; usually only \emph{four} parameters are used
to fit a graph with millions of nodes.

In order to enrich the expressive power of the model \citet{KimLes10}
recently proposed a generalization of KPGM, which is named
Multiplicative Attribute Graph Model (MAGM). The advantage of MAGM
over KPGM has been argued from both theoretical \citep{KimLes10} and
empirical \citep{KimLes11} perspectives.

No matter how attractive such a generalization is in terms of modeling,
we still need to ask \emph{does the new model have efficient algorithms
  for inference and sampling?}  The inference part of this question was
studied by \citet{KimLes11}, while sampling part was partially addressed
by \citet{YunVis12}. In this paper, we further investigate the 
sampling issue.

It is straightforward to sample a graph from a MAGM in $\Theta\rbr{n^2}$
time, where $n$ is the number of nodes. Of course, such a na\"{\i}ve
algorithm does not scale to large graphs. Therefore, \citet{YunVis12}
suggested an algorithm which first samples $O\rbr{\rbr{\log_{2} n}^2}$
graphs from a KPGM and \emph{quilts} relevant parts of the sampled
graphs together to generate a \emph{single} sample from the MAGM. Since
approximate sampling from KPGM takes expected $O\rbr{e_{K} \log_{2} n}$
time, where $e_{K}$ is the expected number of edges in the KPGM, the
quilting algorithm runs in $O\rbr{ \rbr{\log_2 n}^{3} e_{K}}$ time with
high probability. The unsatisfactory aspect of the approach of
\citet{YunVis12}, however, is that the complexity bound holds only when
certain technical conditions are met.

On the other hand, for the most commonly used parameter settings (see
Section \ref{sec:TimeComplexity}) our algorithm runs in
$O\rbr{\rbr{\log_{2} n}^3 \rbr{e_{K} + e_{M}}}$ time with high
probability, where $e_{M}$ is the expected number of edges in the
MAGM. When the technical conditions of \citet{YunVis12} are met, then
$e_{M} = e_{K}$. Therefore, our method extends the \emph{best} time
complexity of \citet{YunVis12} to a larger fraction of parameter
space. Not only is our algorithm theoretically more interesting, we also
show that it empirically outperforms the previous algorithm in sampling
sparse graphs.

To design our algorithm, we first define a stochastic
\emph{ball-dropping process} (BDP) (\citet{ChaZhaFal2004},
\citet{GroSulPoo2010} and \citet{GleOwe2011}).  Although a special case
of BDP was already introduced as an approximate sampling algorithm for
KPGM \citep{LesChaKleFaletal10}, to the best of our knowledge neither
\emph{why} such an approximation works nor \emph{what} is the actual
distribution this process is sampling from has been addressed so far.

Our rigorous treatment of these problems enables us to clarify the
rational behind a BDP approximation of KPGM (Section~\ref{sec:bdp}), and
design an efficient sampling algorithm for MAGM
(Section~\ref{sec:algorithm}). We let BDP to \emph{propose} candidate
edges, and then \emph{reject} some of them with certain probability to
match the actual MAGM. This is the classic accept-reject sampling scheme
for sampling distributions. The main technical challenge which we
address in this paper is to show that the proposal distribution
compactly bounds the target distribution, so that we can guarantee the
efficiency of the algorithm.

\section{Notation and Preliminaries}
\label{sec:Notation}

We use upper-case letters for matrices (\eg $A$). Sets are denoted by
upper-case calligraphic letters (\eg $\Ecal$). We use Greek symbols for
parameters (\eg $\mu$), and integers are denoted in lower-case (\eg
$a, b, i, j$).

A directed graph is an ordered set $(\Vcal,\Ecal)$, where
$\Vcal$ is the set of nodes $\Vcal=\{1,2, \ldots, n\}$, and $\Ecal$ is
the set of edges $\Ecal \subset \Vcal \!\times\! \Vcal$.  We say that
there is an edge from node $i$ to $j$ when $(i,j)\in
\Ecal$. Furthermore, for each edge $(i,j) \in \Ecal$, $i$ and $j$ are
called source and target node of the edge, respectively. Note that
although we mainly discuss directed graphs in this paper, most of our
ideas can be straightforwardly applied to the case of undirected graphs.

It is convenient to describe a graph in terms of its $n \times n$
adjacency matrix $A$ where the $(i,j)$-th entry $A_{ij}$ of $A$ denotes
the number of edges from node $i$ to $j$. When there exists at most one
edge between every $(i,j)$ pair, i.e., $A_{ij} \leq 1$ for all $i, j$,
then we call it a \emph{simple} graph. On the other hand if multiple
edges are allowed then it is called a \emph{multi}-graph. In either
case, $\abr{\Ecal}$, the number of edges in the graph, is equal to
$\sum_{i,j=1}^n A_{ij}$.

The Kronecker multiplication of matrices is defined as follows
\citep{Bernstein05}.
\begin{definition}
  \label{def:kron}
  Given real matrices $X \in \RR^{n \times m}$ and $Y \in \RR^{p
    \times q}$, the Kronecker product $X \otimes Y \in \RR^{np \times
    mq}$ is
  \begin{align*}
    X \otimes Y := \mymatrix{cccc}{X_{11} Y & X_{12} Y & \ldots & X_{1m} Y \\
      \vdots & \vdots & \vdots & \vdots \\
      X_{n1} Y & X_{n2} Y & \ldots & X_{nm} Y}.
  \end{align*}
  The $k$-th Kronecker power $X^{\sbr{k}}$ is $\otimes_{i=1}^{k} X$.
\end{definition}

\subsection{Kronecker Product Graph Model (KPGM)}

\label{sec:skg}

The Kronecker Product Graph Model (KPGM) of \citet{LesChaKleFaletal10}
is usually parametrized by a $2 \times 2$ initiator matrix
\begin{align}
  \label{eq:theta-def}
  \Theta := \mymatrix{cc}{
    \theta_{00} & \theta_{01} \\
    \theta_{10} & \theta_{11} },
\end{align}
with each $\theta_{ij} \in \sbr{0, 1}$, and additional size parameter $d
\in \mathbb{Z}^+$. Using Kronecker multiplication, we construct a $2^d
\times 2^d$ matrix $\Gamma$ from $\Theta$:
\begin{align}
  \label{eq:kpgm_p1}
  \Gamma := \Theta^{\sbr{d}} = \underbrace{\Theta \otimes \Theta \otimes
    \ldots \otimes \Theta}_{d \text{ times}}.
\end{align}
$\Gamma$ is called an \emph{edge probability matrix}, because under the KPGM
the probability of observing an edge from node $i$ to node $j$ is simply
$\Gamma_{ij}$ (see Figure~\ref{fig:ball_drop}). From an adjacency matrix
point of view each $A_{ij}$ is an independent Bernoulli random variable
with $\PP\sbr{A_{ij} = 1} = \Gamma_{ij}$.

Note that one can make the model more general by using multiple
initiator matrices $\Theta^{(1)}, \Theta^{(2)}, \ldots, \Theta^{(d)}$
rather than a single matrix. In this case, the definition of edge
probability matrix $\Gamma$ is modified to 
\begin{align}
  \label{eq:kpgm_p2}
  \Gamma := \Theta^{(1)} \otimes \Theta^{(2)} \otimes \cdots \otimes
  \Theta^{(d)}.
\end{align}
In this paper we will adopt the more general setting
\eqref{eq:kpgm_p2}. For notational convenience, we stack these
initiator matrices to form the parameter array
\begin{align}
  \label{eq:thetat-def}
  \Thetat := \rbr{ \Theta^{(1)}, \Theta^{(2)}, \ldots,
  \Theta^{(d)}}.
\end{align}
Also, $\theta_{ab}^{(k)}$ denotes $(a+1, b+1)$-th entry of
$\Theta^{(k)}$. Given these parameters, the expected number of edges
$e_K$ of KPGM can be calculated using
\begin{align}
  e_K = \sum_{i,j=1}^n \Gamma_{ij} =
  \prod_{k=1}^d \rbr{\sum_{0 \leq a,b \leq 1} \theta_{ab}^{(k)}}.
  \label{eq:skg_edgenum}
\end{align}

\subsection{Multiplicative Attribute Graph Model (MAGM)}

\label{sec:mag}

An alternative way to view KPGM is as follows: associate the $i$-th node
with a bit-vector $b(i)$ of length $d$ such that $b_{k}(i)$ is the
$k$-th digit of integer $(i-1)$ in its binary representation. Then one
can verify that the $(i,j)$-th entry of the edge probability matrix
$\Gamma$ in \eqref{eq:kpgm_p2} can be written as
\begin{align}
  \label{eq:kpgm_prob}
  \Gamma_{ij} = \prod_{k=1}^d \theta^{(k)}_{b_k(i) \; b_k(j)}.
\end{align}
Under this interpretation, one may consider $b_k(i) = 1$ (resp.\
$b_k(i) = 0$) as denoting the presence (resp.\ absence) of the $k$-th
attribute in node $i$. The factor $\theta_{b_k(i) \; b_k(j)}^{(k)}$
denotes the probability of an edge between nodes $i$ and $j$ based on
the value of their $k$-th attribute. The attributes are assumed
independent, and therefore the overall probability of an edge between
$i$ and $j$ is just the product of $\theta_{b_k(i) \;
  b_k(j)}^{(k)}$'s.

The Multiplicative Attribute Graph Model (MAGM) of \citet{KimLes10} is
also obtained by associating a bit-vector $f(i)$ with a node
$i$. However, $f(i)$ need not be the binary representation of $(i-1)$
as was the case in the KPGM. In fact, the number of nodes $n$
need not even be equal to $2^d$. We simply assume that $f_{k}(i)$ is a
Bernoulli random variable with $\PP\sbr{f_{k}(i)=1} = \mu^{(k)}$. In
addition to $\Thetat$ defined in \eqref{eq:thetat-def}, the model now
has additional parameters $\mut := \rbr{\mu^{(1)}, \mu^{(2)}, \ldots,
  \mu^{(d)}}$, and the $(i,j)$-th entry of the edge probability matrix
$\Psi$ is written as
\begin{align}
  \label{eq:mag_prob}
  \Psi_{ij} = \prod_{k=1}^d \theta^{(k)}_{f_k(i) \; f_k(j)}.
\end{align}
The expected number of edges under this model will be denoted $e_M$,
and can be calculated using 
\begin{align}
  e_M = n^2 \cdot \prod_{k=1}^d \rbr{\sum_{0 \leq a,b \leq 1}
    \mu^{a+b}\rbr{1-\mu}^{2-a-b} \theta_{ab}^{(k)}} .
  \label{eq:mag_edgenum}
\end{align}
Note that when $\mu^{(1)} = \mu^{(2)} = \cdots = \mu^{(d)} = 0.5$, we
have $e_M = e_K$ (see Figure~\ref{fig:e_mk_mm}).

\section{Ball-Dropping Process (BDP)}
\label{sec:bdp}

A na\"{\i}ve but exact method of sampling from KPGM is to generate every
entry of adjacency matrix $A$ individually. Of course, such an approach
requires $\Theta\rbr{n^2}$ computation and does not scale to large
graphs.  Alternatively, \citet{LesChaKleFaletal10} suggest the following
stochastic process as an approximate but efficient sampling algorithm
(see Figure \ref{fig:ball_drop}):
\begin{itemize}
\item First, sample the number of edges $\abr{\Ecal}$ from a Poisson
  distribution with parameter $e_K$\footnote{Originally
    \citet{LesChaKleFaletal10} used the normal distribution, but Poisson
    is a very close approximation to the normal distribution especially
    when the number of expected edges is a large
    number 
    (Chapter 1.18, \citet{Dasgupta2011}).
  }.
\item The problem of sampling each individual edge is then converted to
  the problem of locating the position of a ``ball'' which will be
  dropped on a $2^d \times 2^d$ grid. The probability of the ball being
  located at coordinate $(i,j)$ is proportional to $\Gamma_{ij}$. This
  problem can be solved in $O\rbr{ d }$ time by employing a
  divide-and-conquer strategy \citep{LesChaKleFaletal10}. See
  Figure~\ref{fig:ball_drop} for a graphical illustration, and
  Algorithm~\ref{alg:kpgmsample} in Appendix~\ref{sec:pseudo_code} for
  the pseudo-code.
\end{itemize}
\begin{figure}
  \begin{center}
    \begin{tikzpicture}[scale=1.0]
      \draw[draw, fill=black!6] (0.000000,2.400000) rectangle (0.300000, 2.100000);
\draw[draw, fill=black!11] (0.000000,2.100000) rectangle (0.300000, 1.800000);
\draw[draw, fill=black!11] (0.000000,1.800000) rectangle (0.300000, 1.500000);
\draw[draw, fill=black!19] (0.000000,1.500000) rectangle (0.300000, 1.200000);
\draw[draw, fill=black!11] (0.000000,1.200000) rectangle (0.300000, 0.900000);
\draw[draw, fill=black!19] (0.000000,0.900000) rectangle (0.300000, 0.600000);
\draw[draw, fill=black!19] (0.000000,0.600000) rectangle (0.300000, 0.300000);
\draw[draw, fill=black!34] (0.000000,0.300000) rectangle (0.300000, 0.000000);
\draw[draw, fill=black!11] (0.300000,2.400000) rectangle (0.600000, 2.100000);
\draw[draw, fill=black!14] (0.300000,2.100000) rectangle (0.600000, 1.800000);
\draw[draw, fill=black!19] (0.300000,1.800000) rectangle (0.600000, 1.500000);
\draw[draw, fill=black!25] (0.300000,1.500000) rectangle (0.600000, 1.200000);
\draw[draw, fill=black!19] (0.300000,1.200000) rectangle (0.600000, 0.900000);
\draw[draw, fill=black!25] (0.300000,0.900000) rectangle (0.600000, 0.600000);
\draw[draw, fill=black!34] (0.300000,0.600000) rectangle (0.600000, 0.300000);
\draw[draw, fill=black!44] (0.300000,0.300000) rectangle (0.600000, 0.000000);
\draw[draw, fill=black!11] (0.600000,2.400000) rectangle (0.900000, 2.100000);
\draw[draw, fill=black!19] (0.600000,2.100000) rectangle (0.900000, 1.800000);
\draw[draw, fill=black!14] (0.600000,1.800000) rectangle (0.900000, 1.500000);
\draw[draw, fill=black!25] (0.600000,1.500000) rectangle (0.900000, 1.200000);
\draw[draw, fill=black!19] (0.600000,1.200000) rectangle (0.900000, 0.900000);
\draw[draw, fill=black!34] (0.600000,0.900000) rectangle (0.900000, 0.600000);
\draw[draw, fill=black!25] (0.600000,0.600000) rectangle (0.900000, 0.300000);
\draw[draw, fill=black!44] (0.600000,0.300000) rectangle (0.900000, 0.000000);
\draw[draw, fill=black!19] (0.900000,2.400000) rectangle (1.200000, 2.100000);
\draw[draw, fill=black!25] (0.900000,2.100000) rectangle (1.200000, 1.800000);
\draw[draw, fill=black!25] (0.900000,1.800000) rectangle (1.200000, 1.500000);
\draw[draw, fill=black!32] (0.900000,1.500000) rectangle (1.200000, 1.200000);
\draw[draw, fill=black!34] (0.900000,1.200000) rectangle (1.200000, 0.900000);
\draw[draw, fill=black!44] (0.900000,0.900000) rectangle (1.200000, 0.600000);
\draw[draw, fill=black!44] (0.900000,0.600000) rectangle (1.200000, 0.300000);
\draw[draw, fill=black!56] (0.900000,0.300000) rectangle (1.200000, 0.000000);
\draw[draw, fill=black!11] (1.200000,2.400000) rectangle (1.500000, 2.100000);
\draw[draw, fill=black!19] (1.200000,2.100000) rectangle (1.500000, 1.800000);
\draw[draw, fill=black!19] (1.200000,1.800000) rectangle (1.500000, 1.500000);
\draw[draw, fill=black!34] (1.200000,1.500000) rectangle (1.500000, 1.200000);
\draw[draw, fill=black!14] (1.200000,1.200000) rectangle (1.500000, 0.900000);
\draw[draw, fill=black!25] (1.200000,0.900000) rectangle (1.500000, 0.600000);
\draw[draw, fill=black!25] (1.200000,0.600000) rectangle (1.500000, 0.300000);
\draw[draw, fill=black!44] (1.200000,0.300000) rectangle (1.500000, 0.000000);
\draw[draw, fill=black!19] (1.500000,2.400000) rectangle (1.800000, 2.100000);
\draw[draw, fill=black!25] (1.500000,2.100000) rectangle (1.800000, 1.800000);
\draw[draw, fill=black!34] (1.500000,1.800000) rectangle (1.800000, 1.500000);
\draw[draw, fill=black!44] (1.500000,1.500000) rectangle (1.800000, 1.200000);
\draw[draw, fill=black!25] (1.500000,1.200000) rectangle (1.800000, 0.900000);
\draw[draw, fill=black!32] (1.500000,0.900000) rectangle (1.800000, 0.600000);
\draw[draw, fill=black!44] (1.500000,0.600000) rectangle (1.800000, 0.300000);
\draw[draw, fill=black!56] (1.500000,0.300000) rectangle (1.800000, 0.000000);
\draw[draw, fill=black!19] (1.800000,2.400000) rectangle (2.100000, 2.100000);
\draw[draw, fill=black!34] (1.800000,2.100000) rectangle (2.100000, 1.800000);
\draw[draw, fill=black!25] (1.800000,1.800000) rectangle (2.100000, 1.500000);
\draw[draw, fill=black!44] (1.800000,1.500000) rectangle (2.100000, 1.200000);
\draw[draw, fill=black!25] (1.800000,1.200000) rectangle (2.100000, 0.900000);
\draw[draw, fill=black!44] (1.800000,0.900000) rectangle (2.100000, 0.600000);
\draw[draw, fill=black!32] (1.800000,0.600000) rectangle (2.100000, 0.300000);
\draw[draw, fill=black!56] (1.800000,0.300000) rectangle (2.100000, 0.000000);
\draw[draw, fill=black!34] (2.100000,2.400000) rectangle (2.400000, 2.100000);
\draw[draw, fill=black!44] (2.100000,2.100000) rectangle (2.400000, 1.800000);
\draw[draw, fill=black!44] (2.100000,1.800000) rectangle (2.400000, 1.500000);
\draw[draw, fill=black!56] (2.100000,1.500000) rectangle (2.400000, 1.200000);
\draw[draw, fill=black!44] (2.100000,1.200000) rectangle (2.400000, 0.900000);
\draw[draw, fill=black!56] (2.100000,0.900000) rectangle (2.400000, 0.600000);
\draw[draw, fill=black!56] (2.100000,0.600000) rectangle (2.400000, 0.300000);
\draw[draw, fill=black!72] (2.100000,0.300000) rectangle (2.400000, 0.000000);
      \draw[below] (1.25, 0) node{{\small (a)}};
    \end{tikzpicture}
    \;
    \begin{tikzpicture}[scale=1.0]
      \begin{scope}[xshift = 0cm, yshift=0cm]
        \draw[draw, fill=black!6] (0.000000,2.400000) rectangle (0.300000, 2.100000);
\draw[draw, fill=black!11] (0.000000,2.100000) rectangle (0.300000, 1.800000);
\draw[draw, fill=black!11] (0.000000,1.800000) rectangle (0.300000, 1.500000);
\draw[draw, fill=black!19] (0.000000,1.500000) rectangle (0.300000, 1.200000);
\draw[draw, fill=black!11] (0.000000,1.200000) rectangle (0.300000, 0.900000);
\draw[draw, fill=black!19] (0.000000,0.900000) rectangle (0.300000, 0.600000);
\draw[draw, fill=black!19] (0.000000,0.600000) rectangle (0.300000, 0.300000);
\draw[draw, fill=black!34] (0.000000,0.300000) rectangle (0.300000, 0.000000);
\draw[draw, fill=black!11] (0.300000,2.400000) rectangle (0.600000, 2.100000);
\draw[draw, fill=black!14] (0.300000,2.100000) rectangle (0.600000, 1.800000);
\draw[draw, fill=black!19] (0.300000,1.800000) rectangle (0.600000, 1.500000);
\draw[draw, fill=black!25] (0.300000,1.500000) rectangle (0.600000, 1.200000);
\draw[draw, fill=black!19] (0.300000,1.200000) rectangle (0.600000, 0.900000);
\draw[draw, fill=black!25] (0.300000,0.900000) rectangle (0.600000, 0.600000);
\draw[draw, fill=black!34] (0.300000,0.600000) rectangle (0.600000, 0.300000);
\draw[draw, fill=black!44] (0.300000,0.300000) rectangle (0.600000, 0.000000);
\draw[draw, fill=black!11] (0.600000,2.400000) rectangle (0.900000, 2.100000);
\draw[draw, fill=black!19] (0.600000,2.100000) rectangle (0.900000, 1.800000);
\draw[draw, fill=black!14] (0.600000,1.800000) rectangle (0.900000, 1.500000);
\draw[draw, fill=black!25] (0.600000,1.500000) rectangle (0.900000, 1.200000);
\draw[draw, fill=black!19] (0.600000,1.200000) rectangle (0.900000, 0.900000);
\draw[draw, fill=black!34] (0.600000,0.900000) rectangle (0.900000, 0.600000);
\draw[draw, fill=black!25] (0.600000,0.600000) rectangle (0.900000, 0.300000);
\draw[draw, fill=black!44] (0.600000,0.300000) rectangle (0.900000, 0.000000);
\draw[draw, fill=black!19] (0.900000,2.400000) rectangle (1.200000, 2.100000);
\draw[draw, fill=black!25] (0.900000,2.100000) rectangle (1.200000, 1.800000);
\draw[draw, fill=black!25] (0.900000,1.800000) rectangle (1.200000, 1.500000);
\draw[draw, fill=black!32] (0.900000,1.500000) rectangle (1.200000, 1.200000);
\draw[draw, fill=black!34] (0.900000,1.200000) rectangle (1.200000, 0.900000);
\draw[draw, fill=black!44] (0.900000,0.900000) rectangle (1.200000, 0.600000);
\draw[draw, fill=black!44] (0.900000,0.600000) rectangle (1.200000, 0.300000);
\draw[draw, fill=black!56] (0.900000,0.300000) rectangle (1.200000, 0.000000);
\draw[draw, fill=black!11] (1.200000,2.400000) rectangle (1.500000, 2.100000);
\draw[draw, fill=black!19] (1.200000,2.100000) rectangle (1.500000, 1.800000);
\draw[draw, fill=black!19] (1.200000,1.800000) rectangle (1.500000, 1.500000);
\draw[draw, fill=black!34] (1.200000,1.500000) rectangle (1.500000, 1.200000);
\draw[draw, fill=black!14] (1.200000,1.200000) rectangle (1.500000, 0.900000);
\draw[draw, fill=black!25] (1.200000,0.900000) rectangle (1.500000, 0.600000);
\draw[draw, fill=black!25] (1.200000,0.600000) rectangle (1.500000, 0.300000);
\draw[draw, fill=black!44] (1.200000,0.300000) rectangle (1.500000, 0.000000);
\draw[draw, fill=black!19] (1.500000,2.400000) rectangle (1.800000, 2.100000);
\draw[draw, fill=black!25] (1.500000,2.100000) rectangle (1.800000, 1.800000);
\draw[draw, fill=black!34] (1.500000,1.800000) rectangle (1.800000, 1.500000);
\draw[draw, fill=black!44] (1.500000,1.500000) rectangle (1.800000, 1.200000);
\draw[draw, fill=black!25] (1.500000,1.200000) rectangle (1.800000, 0.900000);
\draw[draw, fill=black!32] (1.500000,0.900000) rectangle (1.800000, 0.600000);
\draw[draw, fill=black!44] (1.500000,0.600000) rectangle (1.800000, 0.300000);
\draw[draw, fill=black!56] (1.500000,0.300000) rectangle (1.800000, 0.000000);
\draw[draw, fill=black!19] (1.800000,2.400000) rectangle (2.100000, 2.100000);
\draw[draw, fill=black!34] (1.800000,2.100000) rectangle (2.100000, 1.800000);
\draw[draw, fill=black!25] (1.800000,1.800000) rectangle (2.100000, 1.500000);
\draw[draw, fill=black!44] (1.800000,1.500000) rectangle (2.100000, 1.200000);
\draw[draw, fill=black!25] (1.800000,1.200000) rectangle (2.100000, 0.900000);
\draw[draw, fill=black!44] (1.800000,0.900000) rectangle (2.100000, 0.600000);
\draw[draw, fill=black!32] (1.800000,0.600000) rectangle (2.100000, 0.300000);
\draw[draw, fill=black!56] (1.800000,0.300000) rectangle (2.100000, 0.000000);
\draw[draw, fill=black!34] (2.100000,2.400000) rectangle (2.400000, 2.100000);
\draw[draw, fill=black!44] (2.100000,2.100000) rectangle (2.400000, 1.800000);
\draw[draw, fill=black!44] (2.100000,1.800000) rectangle (2.400000, 1.500000);
\draw[draw, fill=black!56] (2.100000,1.500000) rectangle (2.400000, 1.200000);
\draw[draw, fill=black!44] (2.100000,1.200000) rectangle (2.400000, 0.900000);
\draw[draw, fill=black!56] (2.100000,0.900000) rectangle (2.400000, 0.600000);
\draw[draw, fill=black!56] (2.100000,0.600000) rectangle (2.400000, 0.300000);
\draw[draw, fill=black!72] (2.100000,0.300000) rectangle (2.400000, 0.000000);
        \draw (0, 1.2) edge[red, thick] (2.4, 1.2);
        \draw (1.2, 0) edge[red, thick] (1.2, 2.4);
        \draw[red, draw, thick] (0,0) rectangle (2.4, 2.4);
        \draw[red] (0.6, 1.8) node {{ \bf 0.4}};
        \draw[red] (1.8, 1.8) node {{ \bf 0.7}};
        \draw[red] (0.6, 0.6) node {{ \bf 0.7}};
        \draw[blue] (1.8, 0.6) node {{ \bf 0.9}};
        \draw[below] (1.25, 0) node{{\small (b)}};
      \end{scope}
    \end{tikzpicture}    

    \begin{tikzpicture}[scale=1.0]
      \draw[draw, fill=black!6] (0.000000,2.400000) rectangle (0.300000, 2.100000);
\draw[draw, fill=black!11] (0.000000,2.100000) rectangle (0.300000, 1.800000);
\draw[draw, fill=black!11] (0.000000,1.800000) rectangle (0.300000, 1.500000);
\draw[draw, fill=black!19] (0.000000,1.500000) rectangle (0.300000, 1.200000);
\draw[draw, fill=black!11] (0.000000,1.200000) rectangle (0.300000, 0.900000);
\draw[draw, fill=black!19] (0.000000,0.900000) rectangle (0.300000, 0.600000);
\draw[draw, fill=black!19] (0.000000,0.600000) rectangle (0.300000, 0.300000);
\draw[draw, fill=black!34] (0.000000,0.300000) rectangle (0.300000, 0.000000);
\draw[draw, fill=black!11] (0.300000,2.400000) rectangle (0.600000, 2.100000);
\draw[draw, fill=black!14] (0.300000,2.100000) rectangle (0.600000, 1.800000);
\draw[draw, fill=black!19] (0.300000,1.800000) rectangle (0.600000, 1.500000);
\draw[draw, fill=black!25] (0.300000,1.500000) rectangle (0.600000, 1.200000);
\draw[draw, fill=black!19] (0.300000,1.200000) rectangle (0.600000, 0.900000);
\draw[draw, fill=black!25] (0.300000,0.900000) rectangle (0.600000, 0.600000);
\draw[draw, fill=black!34] (0.300000,0.600000) rectangle (0.600000, 0.300000);
\draw[draw, fill=black!44] (0.300000,0.300000) rectangle (0.600000, 0.000000);
\draw[draw, fill=black!11] (0.600000,2.400000) rectangle (0.900000, 2.100000);
\draw[draw, fill=black!19] (0.600000,2.100000) rectangle (0.900000, 1.800000);
\draw[draw, fill=black!14] (0.600000,1.800000) rectangle (0.900000, 1.500000);
\draw[draw, fill=black!25] (0.600000,1.500000) rectangle (0.900000, 1.200000);
\draw[draw, fill=black!19] (0.600000,1.200000) rectangle (0.900000, 0.900000);
\draw[draw, fill=black!34] (0.600000,0.900000) rectangle (0.900000, 0.600000);
\draw[draw, fill=black!25] (0.600000,0.600000) rectangle (0.900000, 0.300000);
\draw[draw, fill=black!44] (0.600000,0.300000) rectangle (0.900000, 0.000000);
\draw[draw, fill=black!19] (0.900000,2.400000) rectangle (1.200000, 2.100000);
\draw[draw, fill=black!25] (0.900000,2.100000) rectangle (1.200000, 1.800000);
\draw[draw, fill=black!25] (0.900000,1.800000) rectangle (1.200000, 1.500000);
\draw[draw, fill=black!32] (0.900000,1.500000) rectangle (1.200000, 1.200000);
\draw[draw, fill=black!34] (0.900000,1.200000) rectangle (1.200000, 0.900000);
\draw[draw, fill=black!44] (0.900000,0.900000) rectangle (1.200000, 0.600000);
\draw[draw, fill=black!44] (0.900000,0.600000) rectangle (1.200000, 0.300000);
\draw[draw, fill=black!56] (0.900000,0.300000) rectangle (1.200000, 0.000000);
\draw[draw, fill=black!11] (1.200000,2.400000) rectangle (1.500000, 2.100000);
\draw[draw, fill=black!19] (1.200000,2.100000) rectangle (1.500000, 1.800000);
\draw[draw, fill=black!19] (1.200000,1.800000) rectangle (1.500000, 1.500000);
\draw[draw, fill=black!34] (1.200000,1.500000) rectangle (1.500000, 1.200000);
\draw[draw, fill=black!14] (1.200000,1.200000) rectangle (1.500000, 0.900000);
\draw[draw, fill=black!25] (1.200000,0.900000) rectangle (1.500000, 0.600000);
\draw[draw, fill=black!25] (1.200000,0.600000) rectangle (1.500000, 0.300000);
\draw[draw, fill=black!44] (1.200000,0.300000) rectangle (1.500000, 0.000000);
\draw[draw, fill=black!19] (1.500000,2.400000) rectangle (1.800000, 2.100000);
\draw[draw, fill=black!25] (1.500000,2.100000) rectangle (1.800000, 1.800000);
\draw[draw, fill=black!34] (1.500000,1.800000) rectangle (1.800000, 1.500000);
\draw[draw, fill=black!44] (1.500000,1.500000) rectangle (1.800000, 1.200000);
\draw[draw, fill=black!25] (1.500000,1.200000) rectangle (1.800000, 0.900000);
\draw[draw, fill=black!32] (1.500000,0.900000) rectangle (1.800000, 0.600000);
\draw[draw, fill=black!44] (1.500000,0.600000) rectangle (1.800000, 0.300000);
\draw[draw, fill=black!56] (1.500000,0.300000) rectangle (1.800000, 0.000000);
\draw[draw, fill=black!19] (1.800000,2.400000) rectangle (2.100000, 2.100000);
\draw[draw, fill=black!34] (1.800000,2.100000) rectangle (2.100000, 1.800000);
\draw[draw, fill=black!25] (1.800000,1.800000) rectangle (2.100000, 1.500000);
\draw[draw, fill=black!44] (1.800000,1.500000) rectangle (2.100000, 1.200000);
\draw[draw, fill=black!25] (1.800000,1.200000) rectangle (2.100000, 0.900000);
\draw[draw, fill=black!44] (1.800000,0.900000) rectangle (2.100000, 0.600000);
\draw[draw, fill=black!32] (1.800000,0.600000) rectangle (2.100000, 0.300000);
\draw[draw, fill=black!56] (1.800000,0.300000) rectangle (2.100000, 0.000000);
\draw[draw, fill=black!34] (2.100000,2.400000) rectangle (2.400000, 2.100000);
\draw[draw, fill=black!44] (2.100000,2.100000) rectangle (2.400000, 1.800000);
\draw[draw, fill=black!44] (2.100000,1.800000) rectangle (2.400000, 1.500000);
\draw[draw, fill=black!56] (2.100000,1.500000) rectangle (2.400000, 1.200000);
\draw[draw, fill=black!44] (2.100000,1.200000) rectangle (2.400000, 0.900000);
\draw[draw, fill=black!56] (2.100000,0.900000) rectangle (2.400000, 0.600000);
\draw[draw, fill=black!56] (2.100000,0.600000) rectangle (2.400000, 0.300000);
\draw[draw, fill=black!72] (2.100000,0.300000) rectangle (2.400000, 0.000000);
        \draw[red, draw, thick] (1.2,0) rectangle (2.4, 1.2);
        \draw (1.2, 0.6) edge[red, thick] (2.4, 0.6);
        \draw (1.8, 0) edge[red, thick] (1.8, 1.2);
        \draw[red] (1.5, 0.9) node {{ 0.4}};
        \draw[red] (2.1, 0.9) node {{ 0.7}};
        \draw[blue] (1.5, 0.3) node {{ 0.7}};
        \draw[red] (2.1, 0.3) node {{ 0.9}};
        \draw[below] (1.25, 0) node{{\small (c)}};
    \end{tikzpicture}
    \;
    \begin{tikzpicture}[scale=1.0]
      \draw[draw, fill=black!6] (0.000000,2.400000) rectangle (0.300000, 2.100000);
\draw[draw, fill=black!11] (0.000000,2.100000) rectangle (0.300000, 1.800000);
\draw[draw, fill=black!11] (0.000000,1.800000) rectangle (0.300000, 1.500000);
\draw[draw, fill=black!19] (0.000000,1.500000) rectangle (0.300000, 1.200000);
\draw[draw, fill=black!11] (0.000000,1.200000) rectangle (0.300000, 0.900000);
\draw[draw, fill=black!19] (0.000000,0.900000) rectangle (0.300000, 0.600000);
\draw[draw, fill=black!19] (0.000000,0.600000) rectangle (0.300000, 0.300000);
\draw[draw, fill=black!34] (0.000000,0.300000) rectangle (0.300000, 0.000000);
\draw[draw, fill=black!11] (0.300000,2.400000) rectangle (0.600000, 2.100000);
\draw[draw, fill=black!14] (0.300000,2.100000) rectangle (0.600000, 1.800000);
\draw[draw, fill=black!19] (0.300000,1.800000) rectangle (0.600000, 1.500000);
\draw[draw, fill=black!25] (0.300000,1.500000) rectangle (0.600000, 1.200000);
\draw[draw, fill=black!19] (0.300000,1.200000) rectangle (0.600000, 0.900000);
\draw[draw, fill=black!25] (0.300000,0.900000) rectangle (0.600000, 0.600000);
\draw[draw, fill=black!34] (0.300000,0.600000) rectangle (0.600000, 0.300000);
\draw[draw, fill=black!44] (0.300000,0.300000) rectangle (0.600000, 0.000000);
\draw[draw, fill=black!11] (0.600000,2.400000) rectangle (0.900000, 2.100000);
\draw[draw, fill=black!19] (0.600000,2.100000) rectangle (0.900000, 1.800000);
\draw[draw, fill=black!14] (0.600000,1.800000) rectangle (0.900000, 1.500000);
\draw[draw, fill=black!25] (0.600000,1.500000) rectangle (0.900000, 1.200000);
\draw[draw, fill=black!19] (0.600000,1.200000) rectangle (0.900000, 0.900000);
\draw[draw, fill=black!34] (0.600000,0.900000) rectangle (0.900000, 0.600000);
\draw[draw, fill=black!25] (0.600000,0.600000) rectangle (0.900000, 0.300000);
\draw[draw, fill=black!44] (0.600000,0.300000) rectangle (0.900000, 0.000000);
\draw[draw, fill=black!19] (0.900000,2.400000) rectangle (1.200000, 2.100000);
\draw[draw, fill=black!25] (0.900000,2.100000) rectangle (1.200000, 1.800000);
\draw[draw, fill=black!25] (0.900000,1.800000) rectangle (1.200000, 1.500000);
\draw[draw, fill=black!32] (0.900000,1.500000) rectangle (1.200000, 1.200000);
\draw[draw, fill=black!34] (0.900000,1.200000) rectangle (1.200000, 0.900000);
\draw[draw, fill=black!44] (0.900000,0.900000) rectangle (1.200000, 0.600000);
\draw[draw, fill=black!44] (0.900000,0.600000) rectangle (1.200000, 0.300000);
\draw[draw, fill=black!56] (0.900000,0.300000) rectangle (1.200000, 0.000000);
\draw[draw, fill=black!11] (1.200000,2.400000) rectangle (1.500000, 2.100000);
\draw[draw, fill=black!19] (1.200000,2.100000) rectangle (1.500000, 1.800000);
\draw[draw, fill=black!19] (1.200000,1.800000) rectangle (1.500000, 1.500000);
\draw[draw, fill=black!34] (1.200000,1.500000) rectangle (1.500000, 1.200000);
\draw[draw, fill=black!14] (1.200000,1.200000) rectangle (1.500000, 0.900000);
\draw[draw, fill=black!25] (1.200000,0.900000) rectangle (1.500000, 0.600000);
\draw[draw, fill=black!25] (1.200000,0.600000) rectangle (1.500000, 0.300000);
\draw[draw, fill=black!44] (1.200000,0.300000) rectangle (1.500000, 0.000000);
\draw[draw, fill=black!19] (1.500000,2.400000) rectangle (1.800000, 2.100000);
\draw[draw, fill=black!25] (1.500000,2.100000) rectangle (1.800000, 1.800000);
\draw[draw, fill=black!34] (1.500000,1.800000) rectangle (1.800000, 1.500000);
\draw[draw, fill=black!44] (1.500000,1.500000) rectangle (1.800000, 1.200000);
\draw[draw, fill=black!25] (1.500000,1.200000) rectangle (1.800000, 0.900000);
\draw[draw, fill=black!32] (1.500000,0.900000) rectangle (1.800000, 0.600000);
\draw[draw, fill=black!44] (1.500000,0.600000) rectangle (1.800000, 0.300000);
\draw[draw, fill=black!56] (1.500000,0.300000) rectangle (1.800000, 0.000000);
\draw[draw, fill=black!19] (1.800000,2.400000) rectangle (2.100000, 2.100000);
\draw[draw, fill=black!34] (1.800000,2.100000) rectangle (2.100000, 1.800000);
\draw[draw, fill=black!25] (1.800000,1.800000) rectangle (2.100000, 1.500000);
\draw[draw, fill=black!44] (1.800000,1.500000) rectangle (2.100000, 1.200000);
\draw[draw, fill=black!25] (1.800000,1.200000) rectangle (2.100000, 0.900000);
\draw[draw, fill=black!44] (1.800000,0.900000) rectangle (2.100000, 0.600000);
\draw[draw, fill=black!32] (1.800000,0.600000) rectangle (2.100000, 0.300000);
\draw[draw, fill=black!56] (1.800000,0.300000) rectangle (2.100000, 0.000000);
\draw[draw, fill=black!34] (2.100000,2.400000) rectangle (2.400000, 2.100000);
\draw[draw, fill=black!44] (2.100000,2.100000) rectangle (2.400000, 1.800000);
\draw[draw, fill=black!44] (2.100000,1.800000) rectangle (2.400000, 1.500000);
\draw[draw, fill=black!56] (2.100000,1.500000) rectangle (2.400000, 1.200000);
\draw[draw, fill=black!44] (2.100000,1.200000) rectangle (2.400000, 0.900000);
\draw[draw, fill=black!56] (2.100000,0.900000) rectangle (2.400000, 0.600000);
\draw[draw, fill=black!56] (2.100000,0.600000) rectangle (2.400000, 0.300000);
\draw[draw, fill=black!72] (2.100000,0.300000) rectangle (2.400000, 0.000000);
      \draw[red, draw, thick] (1.2,0) rectangle (1.8, 0.6);
      \draw (1.5, 0) edge[red, thick] (1.5, 0.6);
      \draw (1.2, 0.3) edge[red, thick] (1.8, 0.3);
      \draw[red] (1.35, 0.45) node {{\tiny 0.4}};
      \draw[red] (1.65, 0.45) node {{\tiny 0.7}};
      \draw[red] (1.35, 0.15) node {{\tiny 0.7}};
      \draw[blue] (1.65, 0.15) node {{\tiny 0.9}};
      \draw[below] (1.25, 0) node{{\small (d)}};
    \end{tikzpicture}
  \end{center}
  \caption{\emph{(Best viewed in color)} (a) Edge probability matrix $P$
    of a KPGM with parameter $\Theta = (0.4, 0.7; 0.7, 0.9)$ and $d$ =
    3. Darker cells imply a higher probability of observing an edge. (b)
    To locate the position of an edge, the matrix is divided into four
    quadrants, and one of them is chosen randomly with probability
    proportional to the weight given by the $\Theta$ matrix. Here, the
    fourth quadrant is chosen. (c) and (d) The above process continues
    recursively and finally a location in the $8 \times 8$ grid is
    determined for placing an edge. Here nodes 8 and 6 are connected.}
  \label{fig:ball_drop}
\end{figure}
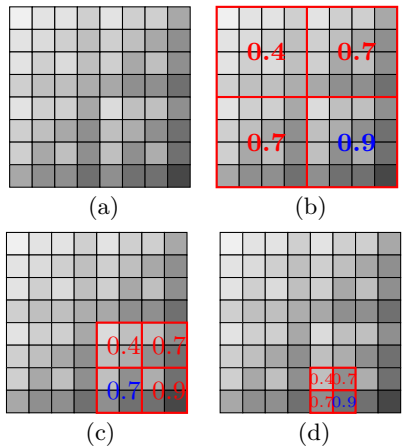
If a graph is sampled from the above process, however, there is a
nonzero probability that the same pair of nodes is sampled multiple
times. Therefore, the process generates multi-graphs while the sample
space of KPGM is simple graphs. The above generative process is called a
\emph{ball-dropping process} (BDP), in order to distinguish it from the
KPGM distribution. Of course, the two are closely related. We show the
following theorem which characterizes the distribution of BDP and
clarifies the connection between the two.
\begin{theorem}[Distribution of BDP] If a multi-graph $G$ is sampled
  from a BDP with parameters $\Thetat$ and $d$, then $A_{ij}$ follows an
  independent Poisson distribution with rate parameter $\Gamma_{ij}$ defined
  by \eqref{eq:kpgm_p2}.
  \label{thm:bdp}
\end{theorem}
\begin{proof}
  See Appendix \ref{sec:ProofTheorrefthm:bdp}. 
\end{proof}
Recall that in the KPGM, each $A_{ij}$ is drawn from an independent
\emph{Bernoulli} distribution, instead of a Poisson distribution. When
the expectation of Bernoulli distribution is close to zero, it is
well-known that the Poisson distribution is a very good approximation to
the Bernoulli distribution (see e.g., Chapter 1.8,
\citet{Dasgupta2011}). To elaborate this point, suppose that a random
variable $X$ follows a Poisson distribution with rate parameter $p$,
while $Y$ follows a Bernoulli distribution with the same parameter $p$.
Then, using the Taylor expansion
\begin{align*}
  \PP\sbr{X = 0} &= \exp(-p) = (1 - p) + O(p^2) \\&= \PP\sbr{Y = 0} + O(p^2),
\end{align*}
In practice we are interested in large sparse graphs, therefore most
$\Gamma_{ij}$ values are close to zero, and the Poisson distribution
provides a good approximation. In fact, this property of the Poisson
distribution is often used in statistical modeling of sparse graphs to
make both analysis tractable and computation more efficient (see \eg
\citet{KarNew11}).

\subsection{Two Observations}
\label{sec:TwoObservations}

Note that $\exp(-p) \geq 1-p$ and consequently the probability of an edge
not being sampled is higher in the BDP than in the KPGM. Consequently,
the BDP generates \emph{sparser} graphs than exact sampling from
KPGM. \citet{LesChaKleFaletal10} observed this and recommend sampling
extra edges to compensate for this effect. Our analysis shows why this
phenomenon occurs.

As the BDP is characterized by a Poisson distribution instead of the
Bernoulli, it only requires non-negativity of its parameters. Therefore,
for a BDP we do not need to enforce the constraint that every
$\theta_{ab}^{(k)}$ parameter is bounded by 1. This extra bit of
generality will be found useful in the next section.

\section{Sampling Algorithm}
\label{sec:algorithm}

In the MAGM, each entry $A_{ij}$ of the adjacency matrix $A$ follows a
Bernoulli distribution with parameter $\Psi_{ij}$.  To efficiently sample
graphs from the model, again we approximate $A_{ij}$ by a Poisson
distribution with the same parameter $\Psi_{ij}$, as discussed in
Section~\ref{sec:bdp}.

A close examination of \eqref{eq:kpgm_prob} and \eqref{eq:mag_prob}
reveals that KPGM and MAGM are very related. The only difference is that
in the case of KPGM the $i$-th node is mapped to the bit vector
corresponding to $(i-1)$ while in the case of MAGM it is mapped to an
integer $c_i$ (not necessarily $(i-1)$) whose bit vector representation
is $f(i)$. We will call $c_i$ the \emph{color} of node $i$ in the
sequel\footnote{\citet{YunVis12} call it \emph{attribute configuration},
  but in our setting we think \emph{color} conveys the idea
  better.}. The concept of color clarifies the connection between KPGM
and MAGM through the following equality
\begin{align}
  \Psi_{ij} = \Gamma_{c_i c_j}.
  \label{eq:q_and_p}
\end{align}

\subsection{Problem Transformation}
\label{ssec:prob_trans}

Let $\Vcal_c$ be the set of nodes with color $0 \leq c \leq n-1$
\begin{align}
  \Vcal_c := \cbr{ i : c_i = c }.
\end{align}
Instead of sampling the adjacency matrix $A$ directly, we will first
generate another matrix $B$, with $B_{cc'}$ defined as
\begin{align}
  B_{cc'} := \sum_{i \in \Vcal_c} \sum_{j \in \Vcal_{c'}} A_{ij}.
  \label{eq:b_cc_def}
\end{align}
\begin{figure}
  \begin{center}
    \begin{tikzpicture}[scale=1.0]
      \draw[draw, fill=black!0] (0.000000,2.400000) rectangle (0.300000, 2.100000);
\draw[draw, fill=black!0] (0.000000,2.100000) rectangle (0.300000, 1.800000);
\draw[draw, fill=black!0] (0.000000,1.800000) rectangle (0.300000, 1.500000);
\draw[draw, fill=black!0] (0.000000,1.500000) rectangle (0.300000, 1.200000);
\draw[draw, fill=black!0] (0.000000,1.200000) rectangle (0.300000, 0.900000);
\draw[draw, fill=black!0] (0.000000,0.900000) rectangle (0.300000, 0.600000);
\draw[draw, fill=black!0] (0.000000,0.600000) rectangle (0.300000, 0.300000);
\draw[draw, fill=black!0] (0.000000,0.300000) rectangle (0.300000, 0.000000);
\draw[draw, fill=black!0] (0.300000,2.400000) rectangle (0.600000, 2.100000);
\draw[draw, fill=black!2] (0.300000,2.100000) rectangle (0.600000, 1.800000);
\draw[draw, fill=black!0] (0.300000,1.800000) rectangle (0.600000, 1.500000);
\draw[draw, fill=black!0] (0.300000,1.500000) rectangle (0.600000, 1.200000);
\draw[draw, fill=black!5] (0.300000,1.200000) rectangle (0.600000, 0.900000);
\draw[draw, fill=black!2] (0.300000,0.900000) rectangle (0.600000, 0.600000);
\draw[draw, fill=black!3] (0.300000,0.600000) rectangle (0.600000, 0.300000);
\draw[draw, fill=black!9] (0.300000,0.300000) rectangle (0.600000, 0.000000);
\draw[draw, fill=black!0] (0.600000,2.400000) rectangle (0.900000, 2.100000);
\draw[draw, fill=black!0] (0.600000,2.100000) rectangle (0.900000, 1.800000);
\draw[draw, fill=black!0] (0.600000,1.800000) rectangle (0.900000, 1.500000);
\draw[draw, fill=black!0] (0.600000,1.500000) rectangle (0.900000, 1.200000);
\draw[draw, fill=black!0] (0.600000,1.200000) rectangle (0.900000, 0.900000);
\draw[draw, fill=black!0] (0.600000,0.900000) rectangle (0.900000, 0.600000);
\draw[draw, fill=black!0] (0.600000,0.600000) rectangle (0.900000, 0.300000);
\draw[draw, fill=black!0] (0.600000,0.300000) rectangle (0.900000, 0.000000);
\draw[draw, fill=black!0] (0.900000,2.400000) rectangle (1.200000, 2.100000);
\draw[draw, fill=black!0] (0.900000,2.100000) rectangle (1.200000, 1.800000);
\draw[draw, fill=black!0] (0.900000,1.800000) rectangle (1.200000, 1.500000);
\draw[draw, fill=black!0] (0.900000,1.500000) rectangle (1.200000, 1.200000);
\draw[draw, fill=black!0] (0.900000,1.200000) rectangle (1.200000, 0.900000);
\draw[draw, fill=black!0] (0.900000,0.900000) rectangle (1.200000, 0.600000);
\draw[draw, fill=black!0] (0.900000,0.600000) rectangle (1.200000, 0.300000);
\draw[draw, fill=black!0] (0.900000,0.300000) rectangle (1.200000, 0.000000);
\draw[draw, fill=black!0] (1.200000,2.400000) rectangle (1.500000, 2.100000);
\draw[draw, fill=black!5] (1.200000,2.100000) rectangle (1.500000, 1.800000);
\draw[draw, fill=black!0] (1.200000,1.800000) rectangle (1.500000, 1.500000);
\draw[draw, fill=black!0] (1.200000,1.500000) rectangle (1.500000, 1.200000);
\draw[draw, fill=black!8] (1.200000,1.200000) rectangle (1.500000, 0.900000);
\draw[draw, fill=black!5] (1.200000,0.900000) rectangle (1.500000, 0.600000);
\draw[draw, fill=black!5] (1.200000,0.600000) rectangle (1.500000, 0.300000);
\draw[draw, fill=black!19] (1.200000,0.300000) rectangle (1.500000, 0.000000);
\draw[draw, fill=black!0] (1.500000,2.400000) rectangle (1.800000, 2.100000);
\draw[draw, fill=black!2] (1.500000,2.100000) rectangle (1.800000, 1.800000);
\draw[draw, fill=black!0] (1.500000,1.800000) rectangle (1.800000, 1.500000);
\draw[draw, fill=black!0] (1.500000,1.500000) rectangle (1.800000, 1.200000);
\draw[draw, fill=black!5] (1.500000,1.200000) rectangle (1.800000, 0.900000);
\draw[draw, fill=black!2] (1.500000,0.900000) rectangle (1.800000, 0.600000);
\draw[draw, fill=black!3] (1.500000,0.600000) rectangle (1.800000, 0.300000);
\draw[draw, fill=black!10] (1.500000,0.300000) rectangle (1.800000, 0.000000);
\draw[draw, fill=black!0] (1.800000,2.400000) rectangle (2.100000, 2.100000);
\draw[draw, fill=black!3] (1.800000,2.100000) rectangle (2.100000, 1.800000);
\draw[draw, fill=black!0] (1.800000,1.800000) rectangle (2.100000, 1.500000);
\draw[draw, fill=black!0] (1.800000,1.500000) rectangle (2.100000, 1.200000);
\draw[draw, fill=black!5] (1.800000,1.200000) rectangle (2.100000, 0.900000);
\draw[draw, fill=black!3] (1.800000,0.900000) rectangle (2.100000, 0.600000);
\draw[draw, fill=black!2] (1.800000,0.600000) rectangle (2.100000, 0.300000);
\draw[draw, fill=black!10] (1.800000,0.300000) rectangle (2.100000, 0.000000);
\draw[draw, fill=black!0] (2.100000,2.400000) rectangle (2.400000, 2.100000);
\draw[draw, fill=black!9] (2.100000,2.100000) rectangle (2.400000, 1.800000);
\draw[draw, fill=black!0] (2.100000,1.800000) rectangle (2.400000, 1.500000);
\draw[draw, fill=black!0] (2.100000,1.500000) rectangle (2.400000, 1.200000);
\draw[draw, fill=black!19] (2.100000,1.200000) rectangle (2.400000, 0.900000);
\draw[draw, fill=black!10] (2.100000,0.900000) rectangle (2.400000, 0.600000);
\draw[draw, fill=black!10] (2.100000,0.600000) rectangle (2.400000, 0.300000);
\draw[draw, fill=black!32] (2.100000,0.300000) rectangle (2.400000, 0.000000);
      \draw[below] (1.25, 0) node{{(a) $\Lambda$}};
    \end{tikzpicture}
    \;
    \begin{tikzpicture}[scale=1.0]
      \draw[draw, fill=black!8] (0.000000,2.400000) rectangle (0.300000, 2.100000);
\draw[draw, fill=black!11] (0.000000,2.100000) rectangle (0.300000, 1.800000);
\draw[draw, fill=black!11] (0.000000,1.800000) rectangle (0.300000, 1.500000);
\draw[draw, fill=black!18] (0.000000,1.500000) rectangle (0.300000, 1.200000);
\draw[draw, fill=black!11] (0.000000,1.200000) rectangle (0.300000, 0.900000);
\draw[draw, fill=black!18] (0.000000,0.900000) rectangle (0.300000, 0.600000);
\draw[draw, fill=black!18] (0.000000,0.600000) rectangle (0.300000, 0.300000);
\draw[draw, fill=black!34] (0.000000,0.300000) rectangle (0.300000, 0.000000);
\draw[draw, fill=black!11] (0.300000,2.400000) rectangle (0.600000, 2.100000);
\draw[draw, fill=black!14] (0.300000,2.100000) rectangle (0.600000, 1.800000);
\draw[draw, fill=black!16] (0.300000,1.800000) rectangle (0.600000, 1.500000);
\draw[draw, fill=black!22] (0.300000,1.500000) rectangle (0.600000, 1.200000);
\draw[draw, fill=black!16] (0.300000,1.200000) rectangle (0.600000, 0.900000);
\draw[draw, fill=black!22] (0.300000,0.900000) rectangle (0.600000, 0.600000);
\draw[draw, fill=black!25] (0.300000,0.600000) rectangle (0.600000, 0.300000);
\draw[draw, fill=black!41] (0.300000,0.300000) rectangle (0.600000, 0.000000);
\draw[draw, fill=black!11] (0.600000,2.400000) rectangle (0.900000, 2.100000);
\draw[draw, fill=black!16] (0.600000,2.100000) rectangle (0.900000, 1.800000);
\draw[draw, fill=black!14] (0.600000,1.800000) rectangle (0.900000, 1.500000);
\draw[draw, fill=black!22] (0.600000,1.500000) rectangle (0.900000, 1.200000);
\draw[draw, fill=black!16] (0.600000,1.200000) rectangle (0.900000, 0.900000);
\draw[draw, fill=black!25] (0.600000,0.900000) rectangle (0.900000, 0.600000);
\draw[draw, fill=black!22] (0.600000,0.600000) rectangle (0.900000, 0.300000);
\draw[draw, fill=black!41] (0.600000,0.300000) rectangle (0.900000, 0.000000);
\draw[draw, fill=black!18] (0.900000,2.400000) rectangle (1.200000, 2.100000);
\draw[draw, fill=black!22] (0.900000,2.100000) rectangle (1.200000, 1.800000);
\draw[draw, fill=black!22] (0.900000,1.800000) rectangle (1.200000, 1.500000);
\draw[draw, fill=black!30] (0.900000,1.500000) rectangle (1.200000, 1.200000);
\draw[draw, fill=black!25] (0.900000,1.200000) rectangle (1.200000, 0.900000);
\draw[draw, fill=black!35] (0.900000,0.900000) rectangle (1.200000, 0.600000);
\draw[draw, fill=black!35] (0.900000,0.600000) rectangle (1.200000, 0.300000);
\draw[draw, fill=black!56] (0.900000,0.300000) rectangle (1.200000, 0.000000);
\draw[draw, fill=black!11] (1.200000,2.400000) rectangle (1.500000, 2.100000);
\draw[draw, fill=black!16] (1.200000,2.100000) rectangle (1.500000, 1.800000);
\draw[draw, fill=black!16] (1.200000,1.800000) rectangle (1.500000, 1.500000);
\draw[draw, fill=black!25] (1.200000,1.500000) rectangle (1.500000, 1.200000);
\draw[draw, fill=black!14] (1.200000,1.200000) rectangle (1.500000, 0.900000);
\draw[draw, fill=black!22] (1.200000,0.900000) rectangle (1.500000, 0.600000);
\draw[draw, fill=black!22] (1.200000,0.600000) rectangle (1.500000, 0.300000);
\draw[draw, fill=black!41] (1.200000,0.300000) rectangle (1.500000, 0.000000);
\draw[draw, fill=black!18] (1.500000,2.400000) rectangle (1.800000, 2.100000);
\draw[draw, fill=black!22] (1.500000,2.100000) rectangle (1.800000, 1.800000);
\draw[draw, fill=black!25] (1.500000,1.800000) rectangle (1.800000, 1.500000);
\draw[draw, fill=black!35] (1.500000,1.500000) rectangle (1.800000, 1.200000);
\draw[draw, fill=black!22] (1.500000,1.200000) rectangle (1.800000, 0.900000);
\draw[draw, fill=black!30] (1.500000,0.900000) rectangle (1.800000, 0.600000);
\draw[draw, fill=black!35] (1.500000,0.600000) rectangle (1.800000, 0.300000);
\draw[draw, fill=black!56] (1.500000,0.300000) rectangle (1.800000, 0.000000);
\draw[draw, fill=black!18] (1.800000,2.400000) rectangle (2.100000, 2.100000);
\draw[draw, fill=black!25] (1.800000,2.100000) rectangle (2.100000, 1.800000);
\draw[draw, fill=black!22] (1.800000,1.800000) rectangle (2.100000, 1.500000);
\draw[draw, fill=black!35] (1.800000,1.500000) rectangle (2.100000, 1.200000);
\draw[draw, fill=black!22] (1.800000,1.200000) rectangle (2.100000, 0.900000);
\draw[draw, fill=black!35] (1.800000,0.900000) rectangle (2.100000, 0.600000);
\draw[draw, fill=black!30] (1.800000,0.600000) rectangle (2.100000, 0.300000);
\draw[draw, fill=black!56] (1.800000,0.300000) rectangle (2.100000, 0.000000);
\draw[draw, fill=black!34] (2.100000,2.400000) rectangle (2.400000, 2.100000);
\draw[draw, fill=black!41] (2.100000,2.100000) rectangle (2.400000, 1.800000);
\draw[draw, fill=black!41] (2.100000,1.800000) rectangle (2.400000, 1.500000);
\draw[draw, fill=black!56] (2.100000,1.500000) rectangle (2.400000, 1.200000);
\draw[draw, fill=black!41] (2.100000,1.200000) rectangle (2.400000, 0.900000);
\draw[draw, fill=black!56] (2.100000,0.900000) rectangle (2.400000, 0.600000);
\draw[draw, fill=black!56] (2.100000,0.600000) rectangle (2.400000, 0.300000);
\draw[draw, fill=black!91] (2.100000,0.300000) rectangle (2.400000, 0.000000);
      \draw[below] (1.25, 0) node{{(b) $\Lambda'$}};
    \end{tikzpicture}
    \;
    \begin{tikzpicture}[scale=1.0]
      \draw[draw, fill=black!0] (0.000000,2.400000) rectangle (0.300000, 2.100000);
\draw[draw, fill=black!0] (0.000000,2.100000) rectangle (0.300000, 1.800000);
\draw[draw, fill=black!0] (0.000000,1.800000) rectangle (0.300000, 1.500000);
\draw[draw, fill=black!0] (0.000000,1.500000) rectangle (0.300000, 1.200000);
\draw[draw, fill=black!0] (0.000000,1.200000) rectangle (0.300000, 0.900000);
\draw[draw, fill=black!0] (0.000000,0.900000) rectangle (0.300000, 0.600000);
\draw[draw, fill=black!0] (0.000000,0.600000) rectangle (0.300000, 0.300000);
\draw[draw, fill=black!0] (0.000000,0.300000) rectangle (0.300000, 0.000000);
\draw[draw, fill=black!0] (0.300000,2.400000) rectangle (0.600000, 2.100000);
\draw[draw, fill=black!15] (0.300000,2.100000) rectangle (0.600000, 1.800000);
\draw[draw, fill=black!0] (0.300000,1.800000) rectangle (0.600000, 1.500000);
\draw[draw, fill=black!0] (0.300000,1.500000) rectangle (0.600000, 1.200000);
\draw[draw, fill=black!30] (0.300000,1.200000) rectangle (0.600000, 0.900000);
\draw[draw, fill=black!11] (0.300000,0.900000) rectangle (0.600000, 0.600000);
\draw[draw, fill=black!11] (0.300000,0.600000) rectangle (0.600000, 0.300000);
\draw[draw, fill=black!23] (0.300000,0.300000) rectangle (0.600000, 0.000000);
\draw[draw, fill=black!0] (0.600000,2.400000) rectangle (0.900000, 2.100000);
\draw[draw, fill=black!0] (0.600000,2.100000) rectangle (0.900000, 1.800000);
\draw[draw, fill=black!0] (0.600000,1.800000) rectangle (0.900000, 1.500000);
\draw[draw, fill=black!0] (0.600000,1.500000) rectangle (0.900000, 1.200000);
\draw[draw, fill=black!0] (0.600000,1.200000) rectangle (0.900000, 0.900000);
\draw[draw, fill=black!0] (0.600000,0.900000) rectangle (0.900000, 0.600000);
\draw[draw, fill=black!0] (0.600000,0.600000) rectangle (0.900000, 0.300000);
\draw[draw, fill=black!0] (0.600000,0.300000) rectangle (0.900000, 0.000000);
\draw[draw, fill=black!0] (0.900000,2.400000) rectangle (1.200000, 2.100000);
\draw[draw, fill=black!0] (0.900000,2.100000) rectangle (1.200000, 1.800000);
\draw[draw, fill=black!0] (0.900000,1.800000) rectangle (1.200000, 1.500000);
\draw[draw, fill=black!0] (0.900000,1.500000) rectangle (1.200000, 1.200000);
\draw[draw, fill=black!0] (0.900000,1.200000) rectangle (1.200000, 0.900000);
\draw[draw, fill=black!0] (0.900000,0.900000) rectangle (1.200000, 0.600000);
\draw[draw, fill=black!0] (0.900000,0.600000) rectangle (1.200000, 0.300000);
\draw[draw, fill=black!0] (0.900000,0.300000) rectangle (1.200000, 0.000000);
\draw[draw, fill=black!0] (1.200000,2.400000) rectangle (1.500000, 2.100000);
\draw[draw, fill=black!30] (1.200000,2.100000) rectangle (1.500000, 1.800000);
\draw[draw, fill=black!0] (1.200000,1.800000) rectangle (1.500000, 1.500000);
\draw[draw, fill=black!0] (1.200000,1.500000) rectangle (1.500000, 1.200000);
\draw[draw, fill=black!61] (1.200000,1.200000) rectangle (1.500000, 0.900000);
\draw[draw, fill=black!23] (1.200000,0.900000) rectangle (1.500000, 0.600000);
\draw[draw, fill=black!23] (1.200000,0.600000) rectangle (1.500000, 0.300000);
\draw[draw, fill=black!47] (1.200000,0.300000) rectangle (1.500000, 0.000000);
\draw[draw, fill=black!0] (1.500000,2.400000) rectangle (1.800000, 2.100000);
\draw[draw, fill=black!11] (1.500000,2.100000) rectangle (1.800000, 1.800000);
\draw[draw, fill=black!0] (1.500000,1.800000) rectangle (1.800000, 1.500000);
\draw[draw, fill=black!0] (1.500000,1.500000) rectangle (1.800000, 1.200000);
\draw[draw, fill=black!23] (1.500000,1.200000) rectangle (1.800000, 0.900000);
\draw[draw, fill=black!9] (1.500000,0.900000) rectangle (1.800000, 0.600000);
\draw[draw, fill=black!9] (1.500000,0.600000) rectangle (1.800000, 0.300000);
\draw[draw, fill=black!18] (1.500000,0.300000) rectangle (1.800000, 0.000000);
\draw[draw, fill=black!0] (1.800000,2.400000) rectangle (2.100000, 2.100000);
\draw[draw, fill=black!11] (1.800000,2.100000) rectangle (2.100000, 1.800000);
\draw[draw, fill=black!0] (1.800000,1.800000) rectangle (2.100000, 1.500000);
\draw[draw, fill=black!0] (1.800000,1.500000) rectangle (2.100000, 1.200000);
\draw[draw, fill=black!23] (1.800000,1.200000) rectangle (2.100000, 0.900000);
\draw[draw, fill=black!9] (1.800000,0.900000) rectangle (2.100000, 0.600000);
\draw[draw, fill=black!9] (1.800000,0.600000) rectangle (2.100000, 0.300000);
\draw[draw, fill=black!18] (1.800000,0.300000) rectangle (2.100000, 0.000000);
\draw[draw, fill=black!0] (2.100000,2.400000) rectangle (2.400000, 2.100000);
\draw[draw, fill=black!23] (2.100000,2.100000) rectangle (2.400000, 1.800000);
\draw[draw, fill=black!0] (2.100000,1.800000) rectangle (2.400000, 1.500000);
\draw[draw, fill=black!0] (2.100000,1.500000) rectangle (2.400000, 1.200000);
\draw[draw, fill=black!47] (2.100000,1.200000) rectangle (2.400000, 0.900000);
\draw[draw, fill=black!18] (2.100000,0.900000) rectangle (2.400000, 0.600000);
\draw[draw, fill=black!18] (2.100000,0.600000) rectangle (2.400000, 0.300000);
\draw[draw, fill=black!35] (2.100000,0.300000) rectangle (2.400000, 0.000000);
      \draw[below] (1.25, 0) node{{(c) $\Lambda \oslash \Lambda'$}};
    \end{tikzpicture}

  \end{center}
  \caption{(a) Poisson parameter matrix $\Lambda$ of target distribution
    $B$. (b) Parameter matrix $\Lambda'$ of proposal distribution
    $B'$. Each entry of $\Lambda'$ must be higher than the corresponding
    entry in $\Lambda$ for $B'$ to be a valid proposal.  (c) The
    acceptance ratio is obtained by Hadamard (element-wise) division of
    $\Lambda$ by $\Lambda'$. The acceptance ratio is high when the gap
    between $\Lambda$ and $\Lambda'$ is small. In all three figures
    darker cells imply higher values and a white cell denotes a zero
    value. Parameters $\Theta = (0.7, 0.85; 0.85, 0.9)$, $d=3$ and
    $\mu=0.7$ was used for these plots. }
  \label{fig:acc_ratio}
\end{figure}
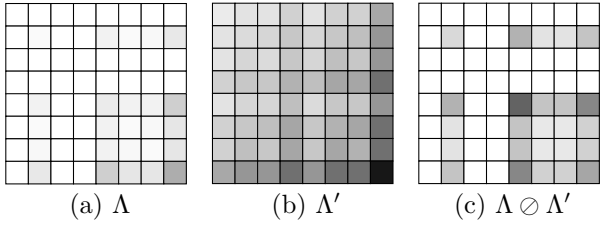
In other words, $B_{cc'}$ is the number of edges from nodes with color
$c$ to nodes with color $c'$. It is easy to verify that each $B_{cc'}$
is a sum of Poisson random variables and hence also follows Poisson
distribution (Chapter 13, \citet{Dasgupta2011}). Let $\Lambda_{cc'}$ be
the rate parameter of the Poisson distribution in $B_{cc'}$, which can
be calculated from \eqref{eq:q_and_p} and \eqref{eq:b_cc_def}
\begin{align}
  \Lambda_{cc'} = \abr{\Vcal_c} \cdot \abr{\Vcal_{c'}} \cdot \Gamma_{cc'}.
  \label{eq:r_cc_def}
\end{align}
Given matrix $B$, it is easy to sample the adjacency matrix $A$.
Uniformly sampling $B_{cc'}$-number of $(i,j)$ pairs in $\Vcal_c \times
\Vcal_{c'}$ for each nonzero entry $B_{cc'}$ of $B$, and incrementing
$A_{ij}$ by 1 for each sampled pair will sample $A$ conditioned on
$B$. An argument similar to the proof of Theorem~\ref{thm:bdp} can be
used to show the validity of such an operation.

That said, now the question is how to efficiently sample $B$. In
Section~\ref{ssec:construction}, we will efficiently construct another
$n \times n$ random matrix $B'$, with $B'_{cc'}$ following an
independent Poisson distribution with parameter $\Lambda'_{cc'}$. $B'$
will \emph{bound} $B$, in the sense that for any $c$ and $c'$ (see
Figure~\ref{fig:acc_ratio}),
\begin{align}
  \Lambda_{cc'} \leq \Lambda'_{cc'}.
  \label{eq:bound_cond}
\end{align}

For each nonzero value of $B'_{cc'}$, sampling from the
binomial distribution of size $B'_{cc'}$ with parameter
$\frac{\Lambda_{cc'}}{\Lambda'_{cc'}}$ will generate a valid $B_{cc'}$. As a
filtered Poisson process is still a Poisson process with an adjusted
parameter value, this step remains valid (Chapter 13.3,
\citet{Dasgupta2011}).

To summarize, we will first generate $B'$, use $B'$ to sample $B$, and
then convert $B$ to $A$. The time complexity of the algorithm is
dominated by the generation of $B'$. See Algorithm~\ref{alg:bdp_sample}
of Appendix~\ref{sec:pseudo_code} for the pseudo-code.

Note that the relation between $B$ and $B'$ is similar to that between
\emph{target} and \emph{proposal} distribution in accept-reject
sampling. While $B$ is the target distribution we \emph{want} to sample,
we first generate a \emph{proposal} $B'$ and correct each entry
$B'_{cc'}$ using acceptance ratio
$\frac{\Lambda_{cc'}}{\Lambda'_{cc'}}$. Just like it is important to
find a good proposal distribution which \emph{compactly bounds} the
target distribution in accept-reject sampling, we need $B'$ which
\emph{compactly bounds} $B$. The remainder of this section is devoted to
show how this can be done.

\subsection{Simple Illustrative Proposal}

To illustrate the idea behind our construction of $B'$, let us first
construct a simple but non-optimal proposal. Let $m$ be the maximum
number of nodes with the same color
\begin{align}
  m := \max_{0 \leq c \leq n-1} \abr{\Vcal_c}.
  \label{eq:m_def}
\end{align}
Using the notation in \eqref{eq:thetat-def}, if one generates a random
matrix $B'$ from BDP with the parameter $\Thetat'$ with each component
$\Theta'^{(k)}$ defined as
\begin{align}
  \Theta'^{(k)} := \rbr{m}^{2/d} \mymatrix{cc}{
    \theta_{00}^{(k)} & \theta_{01}^{(k)} \\
    \theta_{10}^{(k)} & \theta_{11}^{(k)} },
\end{align}
then, by calculation we have $\Lambda'_{cc'} = m^2 \Gamma_{cc'}$. From
definition \eqref{eq:r_cc_def} and \eqref{eq:m_def}, it is obvious that
\eqref{eq:bound_cond} holds
\begin{align}
  \Lambda_{cc'} = \abr{\Vcal_c} \cdot \abr{\Vcal_{c'}} \cdot \Gamma_{cc'} \leq m^2
  \cdot \Gamma_{cc'} = \Lambda'_{cc'},
\end{align}
and hence $B'$ is a \emph{valid} proposal for $B$.

We now investigate the time complexity of sampling $B'$. Since BDP
with parameter $\Thetat$ generates $e_K$ number of edges in
expectation, $B'$ will generate $m^2 \cdot e_K$ edges in
expectation because its BDP parameter is $\Thetat' =
m^{2/d}\Thetat$. As sampling each edge takes $O(d)$ time, the overall
time complexity is $O\rbr{d \cdot m^2 \cdot e_K}$.

If $\mu^{(1)} = \mu^{(2)} = \cdots = \mu^{(d)} = 0.5$ and $n = 2^d$,
\citet{YunVis12} showed that $m \leq \log_2 n$ with high
probability. Therefore, the overall time complexity of sampling is
$O\rbr{d \cdot \rbr{\log_2 n}^2 \cdot e_K}$.

Roughly speaking, the quilting algorithm of \citet{YunVis12} always uses
the same $B'$ irrespective of $\mu^{(k)}$'s. When $\mu^{(k)}$'s are not
exactly equal to 0.5, $m$ is no longer bounded by $\log_2 n$. To resolve
this problem \citet{YunVis12} suggest some heuristics. Instead, we
construct a more careful proposal which adapts to values of $\mu^{(k)}$.

\subsection{Partitioning Colors}

To develop a better proposal $B'$, we define quantities similar to $m$
but bounded by $\log_2 n$ with high probability for general $\mu^{(k)}$
values. To do this, we first partition colors into a set of
\emph{frequent} colors $\Fcal$ and \emph{infrequent} colors $\Ical$
\begin{align}
  \Fcal &:= \cbr{ c : \EE\sbr{\abr{\Vcal_c}} \geq 1}, \\
  \Ical &:= \cbr{ c : \EE\sbr{\abr{\Vcal_c} < 1}} = \cbr{ 0, \ldots, n-1 }
  \backslash \Fcal.
\end{align}
The rational behind this partitioning is as follows: When
$\EE\sbr{\abr{\Vcal_c}} \geq 1$, the variance is smaller than that of
the mean thus $Var\sbr{\abr{\Vcal_c}} \leq \EE\sbr{\abr{\Vcal_c}}$. On
the other hand, when $\EE\abr{V_c} < 1$, then the variance is greater
than that of the mean thus $Var\sbr{\abr{\Vcal_c}} >
\EE\sbr{\abr{\Vcal_c}}$. Therefore, the frequencies of colors in $\Fcal$
and those in $\Ical$ behave very differently, and we need to account for
this. We define
\begin{align}
  m_{\Fcal} &:= \max_{c \in \Fcal}
  \frac{\abr{\Vcal_c}}{\EE\sbr{\abr{\Vcal_c}}}, \;\;\;
  m_{\Ical} := \max_{c \in \Ical} \abr{\Vcal_c} \label{eq:def_m_f}.
\end{align}
\begin{theorem}[Bound of Color Frequencies] With high probability,
  $m_{\Fcal}, m_{\Ical} \leq \log_2 n$.
  \label{thm:color_bound}
\end{theorem}
\begin{proof}
  See Appendix~\ref{sec:proofs}.
\end{proof}

\subsection{Construction of Proposal Distribution}
\label{ssec:construction}

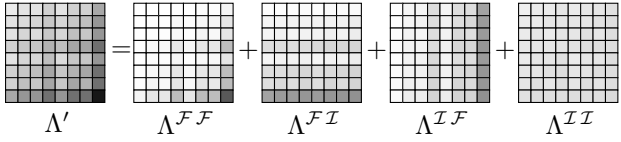
\begin{figure}
  \begin{center}
    \begin{tikzpicture}[scale=0.55]
      \draw[draw, fill=black!8] (0.000000,2.400000) rectangle (0.300000, 2.100000);
\draw[draw, fill=black!11] (0.000000,2.100000) rectangle (0.300000, 1.800000);
\draw[draw, fill=black!11] (0.000000,1.800000) rectangle (0.300000, 1.500000);
\draw[draw, fill=black!18] (0.000000,1.500000) rectangle (0.300000, 1.200000);
\draw[draw, fill=black!11] (0.000000,1.200000) rectangle (0.300000, 0.900000);
\draw[draw, fill=black!18] (0.000000,0.900000) rectangle (0.300000, 0.600000);
\draw[draw, fill=black!18] (0.000000,0.600000) rectangle (0.300000, 0.300000);
\draw[draw, fill=black!34] (0.000000,0.300000) rectangle (0.300000, 0.000000);
\draw[draw, fill=black!11] (0.300000,2.400000) rectangle (0.600000, 2.100000);
\draw[draw, fill=black!14] (0.300000,2.100000) rectangle (0.600000, 1.800000);
\draw[draw, fill=black!16] (0.300000,1.800000) rectangle (0.600000, 1.500000);
\draw[draw, fill=black!22] (0.300000,1.500000) rectangle (0.600000, 1.200000);
\draw[draw, fill=black!16] (0.300000,1.200000) rectangle (0.600000, 0.900000);
\draw[draw, fill=black!22] (0.300000,0.900000) rectangle (0.600000, 0.600000);
\draw[draw, fill=black!25] (0.300000,0.600000) rectangle (0.600000, 0.300000);
\draw[draw, fill=black!41] (0.300000,0.300000) rectangle (0.600000, 0.000000);
\draw[draw, fill=black!11] (0.600000,2.400000) rectangle (0.900000, 2.100000);
\draw[draw, fill=black!16] (0.600000,2.100000) rectangle (0.900000, 1.800000);
\draw[draw, fill=black!14] (0.600000,1.800000) rectangle (0.900000, 1.500000);
\draw[draw, fill=black!22] (0.600000,1.500000) rectangle (0.900000, 1.200000);
\draw[draw, fill=black!16] (0.600000,1.200000) rectangle (0.900000, 0.900000);
\draw[draw, fill=black!25] (0.600000,0.900000) rectangle (0.900000, 0.600000);
\draw[draw, fill=black!22] (0.600000,0.600000) rectangle (0.900000, 0.300000);
\draw[draw, fill=black!41] (0.600000,0.300000) rectangle (0.900000, 0.000000);
\draw[draw, fill=black!18] (0.900000,2.400000) rectangle (1.200000, 2.100000);
\draw[draw, fill=black!22] (0.900000,2.100000) rectangle (1.200000, 1.800000);
\draw[draw, fill=black!22] (0.900000,1.800000) rectangle (1.200000, 1.500000);
\draw[draw, fill=black!30] (0.900000,1.500000) rectangle (1.200000, 1.200000);
\draw[draw, fill=black!25] (0.900000,1.200000) rectangle (1.200000, 0.900000);
\draw[draw, fill=black!35] (0.900000,0.900000) rectangle (1.200000, 0.600000);
\draw[draw, fill=black!35] (0.900000,0.600000) rectangle (1.200000, 0.300000);
\draw[draw, fill=black!56] (0.900000,0.300000) rectangle (1.200000, 0.000000);
\draw[draw, fill=black!11] (1.200000,2.400000) rectangle (1.500000, 2.100000);
\draw[draw, fill=black!16] (1.200000,2.100000) rectangle (1.500000, 1.800000);
\draw[draw, fill=black!16] (1.200000,1.800000) rectangle (1.500000, 1.500000);
\draw[draw, fill=black!25] (1.200000,1.500000) rectangle (1.500000, 1.200000);
\draw[draw, fill=black!14] (1.200000,1.200000) rectangle (1.500000, 0.900000);
\draw[draw, fill=black!22] (1.200000,0.900000) rectangle (1.500000, 0.600000);
\draw[draw, fill=black!22] (1.200000,0.600000) rectangle (1.500000, 0.300000);
\draw[draw, fill=black!41] (1.200000,0.300000) rectangle (1.500000, 0.000000);
\draw[draw, fill=black!18] (1.500000,2.400000) rectangle (1.800000, 2.100000);
\draw[draw, fill=black!22] (1.500000,2.100000) rectangle (1.800000, 1.800000);
\draw[draw, fill=black!25] (1.500000,1.800000) rectangle (1.800000, 1.500000);
\draw[draw, fill=black!35] (1.500000,1.500000) rectangle (1.800000, 1.200000);
\draw[draw, fill=black!22] (1.500000,1.200000) rectangle (1.800000, 0.900000);
\draw[draw, fill=black!30] (1.500000,0.900000) rectangle (1.800000, 0.600000);
\draw[draw, fill=black!35] (1.500000,0.600000) rectangle (1.800000, 0.300000);
\draw[draw, fill=black!56] (1.500000,0.300000) rectangle (1.800000, 0.000000);
\draw[draw, fill=black!18] (1.800000,2.400000) rectangle (2.100000, 2.100000);
\draw[draw, fill=black!25] (1.800000,2.100000) rectangle (2.100000, 1.800000);
\draw[draw, fill=black!22] (1.800000,1.800000) rectangle (2.100000, 1.500000);
\draw[draw, fill=black!35] (1.800000,1.500000) rectangle (2.100000, 1.200000);
\draw[draw, fill=black!22] (1.800000,1.200000) rectangle (2.100000, 0.900000);
\draw[draw, fill=black!35] (1.800000,0.900000) rectangle (2.100000, 0.600000);
\draw[draw, fill=black!30] (1.800000,0.600000) rectangle (2.100000, 0.300000);
\draw[draw, fill=black!56] (1.800000,0.300000) rectangle (2.100000, 0.000000);
\draw[draw, fill=black!34] (2.100000,2.400000) rectangle (2.400000, 2.100000);
\draw[draw, fill=black!41] (2.100000,2.100000) rectangle (2.400000, 1.800000);
\draw[draw, fill=black!41] (2.100000,1.800000) rectangle (2.400000, 1.500000);
\draw[draw, fill=black!56] (2.100000,1.500000) rectangle (2.400000, 1.200000);
\draw[draw, fill=black!41] (2.100000,1.200000) rectangle (2.400000, 0.900000);
\draw[draw, fill=black!56] (2.100000,0.900000) rectangle (2.400000, 0.600000);
\draw[draw, fill=black!56] (2.100000,0.600000) rectangle (2.400000, 0.300000);
\draw[draw, fill=black!91] (2.100000,0.300000) rectangle (2.400000, 0.000000);
      \draw[below] (1.25, 0) node{{$\Lambda'$}};
      \draw[right] (2.325, 1.25) node{{$=$}};
      \begin{scope}[yshift = 0cm, xshift = 3.1cm]
        \draw[draw, fill=black!0] (0.000000,2.400000) rectangle (0.300000, 2.100000);
\draw[draw, fill=black!0] (0.000000,2.100000) rectangle (0.300000, 1.800000);
\draw[draw, fill=black!0] (0.000000,1.800000) rectangle (0.300000, 1.500000);
\draw[draw, fill=black!1] (0.000000,1.500000) rectangle (0.300000, 1.200000);
\draw[draw, fill=black!0] (0.000000,1.200000) rectangle (0.300000, 0.900000);
\draw[draw, fill=black!1] (0.000000,0.900000) rectangle (0.300000, 0.600000);
\draw[draw, fill=black!1] (0.000000,0.600000) rectangle (0.300000, 0.300000);
\draw[draw, fill=black!4] (0.000000,0.300000) rectangle (0.300000, 0.000000);
\draw[draw, fill=black!0] (0.300000,2.400000) rectangle (0.600000, 2.100000);
\draw[draw, fill=black!1] (0.300000,2.100000) rectangle (0.600000, 1.800000);
\draw[draw, fill=black!1] (0.300000,1.800000) rectangle (0.600000, 1.500000);
\draw[draw, fill=black!3] (0.300000,1.500000) rectangle (0.600000, 1.200000);
\draw[draw, fill=black!1] (0.300000,1.200000) rectangle (0.600000, 0.900000);
\draw[draw, fill=black!3] (0.300000,0.900000) rectangle (0.600000, 0.600000);
\draw[draw, fill=black!4] (0.300000,0.600000) rectangle (0.600000, 0.300000);
\draw[draw, fill=black!10] (0.300000,0.300000) rectangle (0.600000, 0.000000);
\draw[draw, fill=black!0] (0.600000,2.400000) rectangle (0.900000, 2.100000);
\draw[draw, fill=black!1] (0.600000,2.100000) rectangle (0.900000, 1.800000);
\draw[draw, fill=black!1] (0.600000,1.800000) rectangle (0.900000, 1.500000);
\draw[draw, fill=black!3] (0.600000,1.500000) rectangle (0.900000, 1.200000);
\draw[draw, fill=black!1] (0.600000,1.200000) rectangle (0.900000, 0.900000);
\draw[draw, fill=black!4] (0.600000,0.900000) rectangle (0.900000, 0.600000);
\draw[draw, fill=black!3] (0.600000,0.600000) rectangle (0.900000, 0.300000);
\draw[draw, fill=black!10] (0.600000,0.300000) rectangle (0.900000, 0.000000);
\draw[draw, fill=black!1] (0.900000,2.400000) rectangle (1.200000, 2.100000);
\draw[draw, fill=black!3] (0.900000,2.100000) rectangle (1.200000, 1.800000);
\draw[draw, fill=black!3] (0.900000,1.800000) rectangle (1.200000, 1.500000);
\draw[draw, fill=black!9] (0.900000,1.500000) rectangle (1.200000, 1.200000);
\draw[draw, fill=black!4] (0.900000,1.200000) rectangle (1.200000, 0.900000);
\draw[draw, fill=black!10] (0.900000,0.900000) rectangle (1.200000, 0.600000);
\draw[draw, fill=black!10] (0.900000,0.600000) rectangle (1.200000, 0.300000);
\draw[draw, fill=black!26] (0.900000,0.300000) rectangle (1.200000, 0.000000);
\draw[draw, fill=black!0] (1.200000,2.400000) rectangle (1.500000, 2.100000);
\draw[draw, fill=black!1] (1.200000,2.100000) rectangle (1.500000, 1.800000);
\draw[draw, fill=black!1] (1.200000,1.800000) rectangle (1.500000, 1.500000);
\draw[draw, fill=black!4] (1.200000,1.500000) rectangle (1.500000, 1.200000);
\draw[draw, fill=black!1] (1.200000,1.200000) rectangle (1.500000, 0.900000);
\draw[draw, fill=black!3] (1.200000,0.900000) rectangle (1.500000, 0.600000);
\draw[draw, fill=black!3] (1.200000,0.600000) rectangle (1.500000, 0.300000);
\draw[draw, fill=black!10] (1.200000,0.300000) rectangle (1.500000, 0.000000);
\draw[draw, fill=black!1] (1.500000,2.400000) rectangle (1.800000, 2.100000);
\draw[draw, fill=black!3] (1.500000,2.100000) rectangle (1.800000, 1.800000);
\draw[draw, fill=black!4] (1.500000,1.800000) rectangle (1.800000, 1.500000);
\draw[draw, fill=black!10] (1.500000,1.500000) rectangle (1.800000, 1.200000);
\draw[draw, fill=black!3] (1.500000,1.200000) rectangle (1.800000, 0.900000);
\draw[draw, fill=black!9] (1.500000,0.900000) rectangle (1.800000, 0.600000);
\draw[draw, fill=black!10] (1.500000,0.600000) rectangle (1.800000, 0.300000);
\draw[draw, fill=black!26] (1.500000,0.300000) rectangle (1.800000, 0.000000);
\draw[draw, fill=black!1] (1.800000,2.400000) rectangle (2.100000, 2.100000);
\draw[draw, fill=black!4] (1.800000,2.100000) rectangle (2.100000, 1.800000);
\draw[draw, fill=black!3] (1.800000,1.800000) rectangle (2.100000, 1.500000);
\draw[draw, fill=black!10] (1.800000,1.500000) rectangle (2.100000, 1.200000);
\draw[draw, fill=black!3] (1.800000,1.200000) rectangle (2.100000, 0.900000);
\draw[draw, fill=black!10] (1.800000,0.900000) rectangle (2.100000, 0.600000);
\draw[draw, fill=black!9] (1.800000,0.600000) rectangle (2.100000, 0.300000);
\draw[draw, fill=black!26] (1.800000,0.300000) rectangle (2.100000, 0.000000);
\draw[draw, fill=black!4] (2.100000,2.400000) rectangle (2.400000, 2.100000);
\draw[draw, fill=black!10] (2.100000,2.100000) rectangle (2.400000, 1.800000);
\draw[draw, fill=black!10] (2.100000,1.800000) rectangle (2.400000, 1.500000);
\draw[draw, fill=black!26] (2.100000,1.500000) rectangle (2.400000, 1.200000);
\draw[draw, fill=black!10] (2.100000,1.200000) rectangle (2.400000, 0.900000);
\draw[draw, fill=black!26] (2.100000,0.900000) rectangle (2.400000, 0.600000);
\draw[draw, fill=black!26] (2.100000,0.600000) rectangle (2.400000, 0.300000);
\draw[draw, fill=black!65] (2.100000,0.300000) rectangle (2.400000, 0.000000);
        \draw[below] (1.25, 0) node{{$\Lambda^{\Fcal\Fcal}$}};
      \end{scope}
      \draw[right] (5.4, 1.25) node{{$+$}};
      \begin{scope}[yshift = 0cm, xshift = 6.2cm]
        \draw[draw, fill=black!1] (0.000000,2.400000) rectangle (0.300000, 2.100000);
\draw[draw, fill=black!4] (0.000000,2.100000) rectangle (0.300000, 1.800000);
\draw[draw, fill=black!4] (0.000000,1.800000) rectangle (0.300000, 1.500000);
\draw[draw, fill=black!13] (0.000000,1.500000) rectangle (0.300000, 1.200000);
\draw[draw, fill=black!4] (0.000000,1.200000) rectangle (0.300000, 0.900000);
\draw[draw, fill=black!13] (0.000000,0.900000) rectangle (0.300000, 0.600000);
\draw[draw, fill=black!13] (0.000000,0.600000) rectangle (0.300000, 0.300000);
\draw[draw, fill=black!36] (0.000000,0.300000) rectangle (0.300000, 0.000000);
\draw[draw, fill=black!1] (0.300000,2.400000) rectangle (0.600000, 2.100000);
\draw[draw, fill=black!4] (0.300000,2.100000) rectangle (0.600000, 1.800000);
\draw[draw, fill=black!5] (0.300000,1.800000) rectangle (0.600000, 1.500000);
\draw[draw, fill=black!13] (0.300000,1.500000) rectangle (0.600000, 1.200000);
\draw[draw, fill=black!5] (0.300000,1.200000) rectangle (0.600000, 0.900000);
\draw[draw, fill=black!13] (0.300000,0.900000) rectangle (0.600000, 0.600000);
\draw[draw, fill=black!15] (0.300000,0.600000) rectangle (0.600000, 0.300000);
\draw[draw, fill=black!39] (0.300000,0.300000) rectangle (0.600000, 0.000000);
\draw[draw, fill=black!1] (0.600000,2.400000) rectangle (0.900000, 2.100000);
\draw[draw, fill=black!5] (0.600000,2.100000) rectangle (0.900000, 1.800000);
\draw[draw, fill=black!4] (0.600000,1.800000) rectangle (0.900000, 1.500000);
\draw[draw, fill=black!13] (0.600000,1.500000) rectangle (0.900000, 1.200000);
\draw[draw, fill=black!5] (0.600000,1.200000) rectangle (0.900000, 0.900000);
\draw[draw, fill=black!15] (0.600000,0.900000) rectangle (0.900000, 0.600000);
\draw[draw, fill=black!13] (0.600000,0.600000) rectangle (0.900000, 0.300000);
\draw[draw, fill=black!39] (0.600000,0.300000) rectangle (0.900000, 0.000000);
\draw[draw, fill=black!2] (0.900000,2.400000) rectangle (1.200000, 2.100000);
\draw[draw, fill=black!5] (0.900000,2.100000) rectangle (1.200000, 1.800000);
\draw[draw, fill=black!5] (0.900000,1.800000) rectangle (1.200000, 1.500000);
\draw[draw, fill=black!14] (0.900000,1.500000) rectangle (1.200000, 1.200000);
\draw[draw, fill=black!6] (0.900000,1.200000) rectangle (1.200000, 0.900000);
\draw[draw, fill=black!16] (0.900000,0.900000) rectangle (1.200000, 0.600000);
\draw[draw, fill=black!16] (0.900000,0.600000) rectangle (1.200000, 0.300000);
\draw[draw, fill=black!41] (0.900000,0.300000) rectangle (1.200000, 0.000000);
\draw[draw, fill=black!1] (1.200000,2.400000) rectangle (1.500000, 2.100000);
\draw[draw, fill=black!5] (1.200000,2.100000) rectangle (1.500000, 1.800000);
\draw[draw, fill=black!5] (1.200000,1.800000) rectangle (1.500000, 1.500000);
\draw[draw, fill=black!15] (1.200000,1.500000) rectangle (1.500000, 1.200000);
\draw[draw, fill=black!4] (1.200000,1.200000) rectangle (1.500000, 0.900000);
\draw[draw, fill=black!13] (1.200000,0.900000) rectangle (1.500000, 0.600000);
\draw[draw, fill=black!13] (1.200000,0.600000) rectangle (1.500000, 0.300000);
\draw[draw, fill=black!39] (1.200000,0.300000) rectangle (1.500000, 0.000000);
\draw[draw, fill=black!2] (1.500000,2.400000) rectangle (1.800000, 2.100000);
\draw[draw, fill=black!5] (1.500000,2.100000) rectangle (1.800000, 1.800000);
\draw[draw, fill=black!6] (1.500000,1.800000) rectangle (1.800000, 1.500000);
\draw[draw, fill=black!16] (1.500000,1.500000) rectangle (1.800000, 1.200000);
\draw[draw, fill=black!5] (1.500000,1.200000) rectangle (1.800000, 0.900000);
\draw[draw, fill=black!14] (1.500000,0.900000) rectangle (1.800000, 0.600000);
\draw[draw, fill=black!16] (1.500000,0.600000) rectangle (1.800000, 0.300000);
\draw[draw, fill=black!41] (1.500000,0.300000) rectangle (1.800000, 0.000000);
\draw[draw, fill=black!2] (1.800000,2.400000) rectangle (2.100000, 2.100000);
\draw[draw, fill=black!6] (1.800000,2.100000) rectangle (2.100000, 1.800000);
\draw[draw, fill=black!5] (1.800000,1.800000) rectangle (2.100000, 1.500000);
\draw[draw, fill=black!16] (1.800000,1.500000) rectangle (2.100000, 1.200000);
\draw[draw, fill=black!5] (1.800000,1.200000) rectangle (2.100000, 0.900000);
\draw[draw, fill=black!16] (1.800000,0.900000) rectangle (2.100000, 0.600000);
\draw[draw, fill=black!14] (1.800000,0.600000) rectangle (2.100000, 0.300000);
\draw[draw, fill=black!41] (1.800000,0.300000) rectangle (2.100000, 0.000000);
\draw[draw, fill=black!2] (2.100000,2.400000) rectangle (2.400000, 2.100000);
\draw[draw, fill=black!7] (2.100000,2.100000) rectangle (2.400000, 1.800000);
\draw[draw, fill=black!7] (2.100000,1.800000) rectangle (2.400000, 1.500000);
\draw[draw, fill=black!17] (2.100000,1.500000) rectangle (2.400000, 1.200000);
\draw[draw, fill=black!7] (2.100000,1.200000) rectangle (2.400000, 0.900000);
\draw[draw, fill=black!17] (2.100000,0.900000) rectangle (2.400000, 0.600000);
\draw[draw, fill=black!17] (2.100000,0.600000) rectangle (2.400000, 0.300000);
\draw[draw, fill=black!43] (2.100000,0.300000) rectangle (2.400000, 0.000000);
        \draw[below] (1.25, 0) node{{$\Lambda^{\Fcal\Ical}$}};
      \end{scope}
      \draw[right] (8.5, 1.25) node{{$+$}};
      \begin{scope}[yshift = 0cm, xshift = 9.3cm]
        \draw[draw, fill=black!1] (0.000000,2.400000) rectangle (0.300000, 2.100000);
\draw[draw, fill=black!1] (0.000000,2.100000) rectangle (0.300000, 1.800000);
\draw[draw, fill=black!1] (0.000000,1.800000) rectangle (0.300000, 1.500000);
\draw[draw, fill=black!2] (0.000000,1.500000) rectangle (0.300000, 1.200000);
\draw[draw, fill=black!1] (0.000000,1.200000) rectangle (0.300000, 0.900000);
\draw[draw, fill=black!2] (0.000000,0.900000) rectangle (0.300000, 0.600000);
\draw[draw, fill=black!2] (0.000000,0.600000) rectangle (0.300000, 0.300000);
\draw[draw, fill=black!2] (0.000000,0.300000) rectangle (0.300000, 0.000000);
\draw[draw, fill=black!4] (0.300000,2.400000) rectangle (0.600000, 2.100000);
\draw[draw, fill=black!4] (0.300000,2.100000) rectangle (0.600000, 1.800000);
\draw[draw, fill=black!5] (0.300000,1.800000) rectangle (0.600000, 1.500000);
\draw[draw, fill=black!5] (0.300000,1.500000) rectangle (0.600000, 1.200000);
\draw[draw, fill=black!5] (0.300000,1.200000) rectangle (0.600000, 0.900000);
\draw[draw, fill=black!5] (0.300000,0.900000) rectangle (0.600000, 0.600000);
\draw[draw, fill=black!6] (0.300000,0.600000) rectangle (0.600000, 0.300000);
\draw[draw, fill=black!7] (0.300000,0.300000) rectangle (0.600000, 0.000000);
\draw[draw, fill=black!4] (0.600000,2.400000) rectangle (0.900000, 2.100000);
\draw[draw, fill=black!5] (0.600000,2.100000) rectangle (0.900000, 1.800000);
\draw[draw, fill=black!4] (0.600000,1.800000) rectangle (0.900000, 1.500000);
\draw[draw, fill=black!5] (0.600000,1.500000) rectangle (0.900000, 1.200000);
\draw[draw, fill=black!5] (0.600000,1.200000) rectangle (0.900000, 0.900000);
\draw[draw, fill=black!6] (0.600000,0.900000) rectangle (0.900000, 0.600000);
\draw[draw, fill=black!5] (0.600000,0.600000) rectangle (0.900000, 0.300000);
\draw[draw, fill=black!7] (0.600000,0.300000) rectangle (0.900000, 0.000000);
\draw[draw, fill=black!13] (0.900000,2.400000) rectangle (1.200000, 2.100000);
\draw[draw, fill=black!13] (0.900000,2.100000) rectangle (1.200000, 1.800000);
\draw[draw, fill=black!13] (0.900000,1.800000) rectangle (1.200000, 1.500000);
\draw[draw, fill=black!14] (0.900000,1.500000) rectangle (1.200000, 1.200000);
\draw[draw, fill=black!15] (0.900000,1.200000) rectangle (1.200000, 0.900000);
\draw[draw, fill=black!16] (0.900000,0.900000) rectangle (1.200000, 0.600000);
\draw[draw, fill=black!16] (0.900000,0.600000) rectangle (1.200000, 0.300000);
\draw[draw, fill=black!17] (0.900000,0.300000) rectangle (1.200000, 0.000000);
\draw[draw, fill=black!4] (1.200000,2.400000) rectangle (1.500000, 2.100000);
\draw[draw, fill=black!5] (1.200000,2.100000) rectangle (1.500000, 1.800000);
\draw[draw, fill=black!5] (1.200000,1.800000) rectangle (1.500000, 1.500000);
\draw[draw, fill=black!6] (1.200000,1.500000) rectangle (1.500000, 1.200000);
\draw[draw, fill=black!4] (1.200000,1.200000) rectangle (1.500000, 0.900000);
\draw[draw, fill=black!5] (1.200000,0.900000) rectangle (1.500000, 0.600000);
\draw[draw, fill=black!5] (1.200000,0.600000) rectangle (1.500000, 0.300000);
\draw[draw, fill=black!7] (1.200000,0.300000) rectangle (1.500000, 0.000000);
\draw[draw, fill=black!13] (1.500000,2.400000) rectangle (1.800000, 2.100000);
\draw[draw, fill=black!13] (1.500000,2.100000) rectangle (1.800000, 1.800000);
\draw[draw, fill=black!15] (1.500000,1.800000) rectangle (1.800000, 1.500000);
\draw[draw, fill=black!16] (1.500000,1.500000) rectangle (1.800000, 1.200000);
\draw[draw, fill=black!13] (1.500000,1.200000) rectangle (1.800000, 0.900000);
\draw[draw, fill=black!14] (1.500000,0.900000) rectangle (1.800000, 0.600000);
\draw[draw, fill=black!16] (1.500000,0.600000) rectangle (1.800000, 0.300000);
\draw[draw, fill=black!17] (1.500000,0.300000) rectangle (1.800000, 0.000000);
\draw[draw, fill=black!13] (1.800000,2.400000) rectangle (2.100000, 2.100000);
\draw[draw, fill=black!15] (1.800000,2.100000) rectangle (2.100000, 1.800000);
\draw[draw, fill=black!13] (1.800000,1.800000) rectangle (2.100000, 1.500000);
\draw[draw, fill=black!16] (1.800000,1.500000) rectangle (2.100000, 1.200000);
\draw[draw, fill=black!13] (1.800000,1.200000) rectangle (2.100000, 0.900000);
\draw[draw, fill=black!16] (1.800000,0.900000) rectangle (2.100000, 0.600000);
\draw[draw, fill=black!14] (1.800000,0.600000) rectangle (2.100000, 0.300000);
\draw[draw, fill=black!17] (1.800000,0.300000) rectangle (2.100000, 0.000000);
\draw[draw, fill=black!36] (2.100000,2.400000) rectangle (2.400000, 2.100000);
\draw[draw, fill=black!39] (2.100000,2.100000) rectangle (2.400000, 1.800000);
\draw[draw, fill=black!39] (2.100000,1.800000) rectangle (2.400000, 1.500000);
\draw[draw, fill=black!41] (2.100000,1.500000) rectangle (2.400000, 1.200000);
\draw[draw, fill=black!39] (2.100000,1.200000) rectangle (2.400000, 0.900000);
\draw[draw, fill=black!41] (2.100000,0.900000) rectangle (2.400000, 0.600000);
\draw[draw, fill=black!41] (2.100000,0.600000) rectangle (2.400000, 0.300000);
\draw[draw, fill=black!43] (2.100000,0.300000) rectangle (2.400000, 0.000000);
        \draw[below] (1.25, 0) node{{$\Lambda^{\Ical\Fcal}$}};
      \end{scope}
      \draw[right] (11.6, 1.25) node{{$+$}};
      \begin{scope}[yshift = 0cm, xshift = 12.4cm]
        \draw[draw, fill=black!6] (0.000000,2.400000) rectangle (0.300000, 2.100000);
\draw[draw, fill=black!8] (0.000000,2.100000) rectangle (0.300000, 1.800000);
\draw[draw, fill=black!8] (0.000000,1.800000) rectangle (0.300000, 1.500000);
\draw[draw, fill=black!10] (0.000000,1.500000) rectangle (0.300000, 1.200000);
\draw[draw, fill=black!8] (0.000000,1.200000) rectangle (0.300000, 0.900000);
\draw[draw, fill=black!10] (0.000000,0.900000) rectangle (0.300000, 0.600000);
\draw[draw, fill=black!10] (0.000000,0.600000) rectangle (0.300000, 0.300000);
\draw[draw, fill=black!12] (0.000000,0.300000) rectangle (0.300000, 0.000000);
\draw[draw, fill=black!8] (0.300000,2.400000) rectangle (0.600000, 2.100000);
\draw[draw, fill=black!8] (0.300000,2.100000) rectangle (0.600000, 1.800000);
\draw[draw, fill=black!10] (0.300000,1.800000) rectangle (0.600000, 1.500000);
\draw[draw, fill=black!10] (0.300000,1.500000) rectangle (0.600000, 1.200000);
\draw[draw, fill=black!10] (0.300000,1.200000) rectangle (0.600000, 0.900000);
\draw[draw, fill=black!10] (0.300000,0.900000) rectangle (0.600000, 0.600000);
\draw[draw, fill=black!12] (0.300000,0.600000) rectangle (0.600000, 0.300000);
\draw[draw, fill=black!13] (0.300000,0.300000) rectangle (0.600000, 0.000000);
\draw[draw, fill=black!8] (0.600000,2.400000) rectangle (0.900000, 2.100000);
\draw[draw, fill=black!10] (0.600000,2.100000) rectangle (0.900000, 1.800000);
\draw[draw, fill=black!8] (0.600000,1.800000) rectangle (0.900000, 1.500000);
\draw[draw, fill=black!10] (0.600000,1.500000) rectangle (0.900000, 1.200000);
\draw[draw, fill=black!10] (0.600000,1.200000) rectangle (0.900000, 0.900000);
\draw[draw, fill=black!12] (0.600000,0.900000) rectangle (0.900000, 0.600000);
\draw[draw, fill=black!10] (0.600000,0.600000) rectangle (0.900000, 0.300000);
\draw[draw, fill=black!13] (0.600000,0.300000) rectangle (0.900000, 0.000000);
\draw[draw, fill=black!10] (0.900000,2.400000) rectangle (1.200000, 2.100000);
\draw[draw, fill=black!10] (0.900000,2.100000) rectangle (1.200000, 1.800000);
\draw[draw, fill=black!10] (0.900000,1.800000) rectangle (1.200000, 1.500000);
\draw[draw, fill=black!11] (0.900000,1.500000) rectangle (1.200000, 1.200000);
\draw[draw, fill=black!12] (0.900000,1.200000) rectangle (1.200000, 0.900000);
\draw[draw, fill=black!13] (0.900000,0.900000) rectangle (1.200000, 0.600000);
\draw[draw, fill=black!13] (0.900000,0.600000) rectangle (1.200000, 0.300000);
\draw[draw, fill=black!13] (0.900000,0.300000) rectangle (1.200000, 0.000000);
\draw[draw, fill=black!8] (1.200000,2.400000) rectangle (1.500000, 2.100000);
\draw[draw, fill=black!10] (1.200000,2.100000) rectangle (1.500000, 1.800000);
\draw[draw, fill=black!10] (1.200000,1.800000) rectangle (1.500000, 1.500000);
\draw[draw, fill=black!12] (1.200000,1.500000) rectangle (1.500000, 1.200000);
\draw[draw, fill=black!8] (1.200000,1.200000) rectangle (1.500000, 0.900000);
\draw[draw, fill=black!10] (1.200000,0.900000) rectangle (1.500000, 0.600000);
\draw[draw, fill=black!10] (1.200000,0.600000) rectangle (1.500000, 0.300000);
\draw[draw, fill=black!13] (1.200000,0.300000) rectangle (1.500000, 0.000000);
\draw[draw, fill=black!10] (1.500000,2.400000) rectangle (1.800000, 2.100000);
\draw[draw, fill=black!10] (1.500000,2.100000) rectangle (1.800000, 1.800000);
\draw[draw, fill=black!12] (1.500000,1.800000) rectangle (1.800000, 1.500000);
\draw[draw, fill=black!13] (1.500000,1.500000) rectangle (1.800000, 1.200000);
\draw[draw, fill=black!10] (1.500000,1.200000) rectangle (1.800000, 0.900000);
\draw[draw, fill=black!11] (1.500000,0.900000) rectangle (1.800000, 0.600000);
\draw[draw, fill=black!13] (1.500000,0.600000) rectangle (1.800000, 0.300000);
\draw[draw, fill=black!13] (1.500000,0.300000) rectangle (1.800000, 0.000000);
\draw[draw, fill=black!10] (1.800000,2.400000) rectangle (2.100000, 2.100000);
\draw[draw, fill=black!12] (1.800000,2.100000) rectangle (2.100000, 1.800000);
\draw[draw, fill=black!10] (1.800000,1.800000) rectangle (2.100000, 1.500000);
\draw[draw, fill=black!13] (1.800000,1.500000) rectangle (2.100000, 1.200000);
\draw[draw, fill=black!10] (1.800000,1.200000) rectangle (2.100000, 0.900000);
\draw[draw, fill=black!13] (1.800000,0.900000) rectangle (2.100000, 0.600000);
\draw[draw, fill=black!11] (1.800000,0.600000) rectangle (2.100000, 0.300000);
\draw[draw, fill=black!13] (1.800000,0.300000) rectangle (2.100000, 0.000000);
\draw[draw, fill=black!12] (2.100000,2.400000) rectangle (2.400000, 2.100000);
\draw[draw, fill=black!13] (2.100000,2.100000) rectangle (2.400000, 1.800000);
\draw[draw, fill=black!13] (2.100000,1.800000) rectangle (2.400000, 1.500000);
\draw[draw, fill=black!13] (2.100000,1.500000) rectangle (2.400000, 1.200000);
\draw[draw, fill=black!13] (2.100000,1.200000) rectangle (2.400000, 0.900000);
\draw[draw, fill=black!13] (2.100000,0.900000) rectangle (2.400000, 0.600000);
\draw[draw, fill=black!13] (2.100000,0.600000) rectangle (2.400000, 0.300000);
\draw[draw, fill=black!14] (2.100000,0.300000) rectangle (2.400000, 0.000000);
        \draw[below] (1.25, 0) node{{$\Lambda^{\Ical\Ical}$}};
      \end{scope}
    \end{tikzpicture}
  \end{center}
  \caption{Decomposition of $\Lambda'$ into $\Lambda_{\Ical, \Ical}$, $\Lambda_{\Fcal,
      \Fcal}$, $\Lambda_{\Fcal, \Ical}$ and $\Lambda_{\Ical, \Fcal}$ . Parameters
    $\Theta = (0.7, 0.85; 0.85, 0.9)$, $d=3$ and $\mu=0.7$ was used. It
    can be seen that the values of $\Lambda^{\Fcal\Fcal}$ are concentrated on
    highly probable pairs, while the values of $\Lambda^{\Ical\Ical}$ are
    relatively spread out.}
  \label{fig:b_decomposition}
\end{figure}

Finally, we construct the proposal distribution. The matrix $B'$ is
the sum of four different BDP matrices
\begin{align}
  B' = B^{(\Fcal\Fcal)} + B^{(\Fcal\Ical)} + B^{(\Ical\Fcal)} + B^{(\Ical\Ical)}.
\end{align}
Intuitively, $B^{(\Fcal\Fcal)}$ concentrates on covering entries of
$B$ between frequent colors, while $B^{(\Ical\Ical)}$ spreads out its
parameters to ensure that every entry of $B$ is properly covered. On the
other hand, $B^{(\Fcal\Ical)}$ and $B^{(\Ical\Fcal)}$ covers entries
between a frequent color and other colors.
Figure~\ref{fig:b_decomposition} visualizes the effect of each
component.

For $\Acal, \Bcal \in \cbr{\Fcal, \Ical}$, let
$\Thetat^{(\Acal\Bcal)}$ and $d$ be parameters of BDP $B^{(\Acal,
  \Bcal)}$. Following notation in \eqref{eq:thetat-def} again, the
$k$-th component of these matrices are defined as
\begin{align}
  \Theta'^{(\Fcal\Fcal)(k)} &:= \rbr{n \, m_{\Fcal}}^{\frac 2 d} \cdot
  \nonumber \\
  &\;\;\;\mymatrix{cc}{
    \rbr{1-\mu^{(k)}}^2 \theta_{00}^{(k)} & \rbr{1-\mu^{(k)}} \mu^{(k)} \theta_{01}^{(k)} \\
    \mu^{(k)} \rbr{1-\mu^{(k)}} \theta_{10}^{(k)} & \rbr{\mu^{(k)}}^2
    \theta_{11}^{(k)} }, \nonumber \\
  \Theta'^{(\Fcal\Ical)(k)} &:= \rbr{n \, m_{\Fcal} \, m_{\Ical}}^{\frac 1 d} \cdot
  \nonumber \\
  &\;\;\;\mymatrix{cc}{
    \rbr{1-\mu^{(k)}} \theta_{00}^{(k)} & \rbr{1-\mu^{(k)}} \theta_{01}^{(k)} \\
    \mu^{(k)} \theta_{10}^{(k)} & \rbr{\mu^{(k)}} \theta_{11}^{(k)} },
  \nonumber \\
  \Theta'^{(\Ical\Fcal)(k)} &:= \rbr{n \, m_{\Ical} \, m_{\Fcal}}^{\frac 1 d}
  \mymatrix{cc}{
    \rbr{1-\mu^{(k)}} \theta_{00}^{(k)} & \mu^{(k)} \theta_{01}^{(k)} \\
    \rbr{1-\mu^{(k)}} \theta_{10}^{(k)} & \mu^{(k)} \theta_{11}^{(k)}
  }, \nonumber \\
  \Theta'^{(\Ical\Ical)(k)} &:= \rbr{m_{\Ical}}^{\frac 2 d}
  \mymatrix{cc}{
    \theta_{00}^{(k)} & \theta_{01}^{(k)} \\
    \theta_{10}^{(k)} & \theta_{11}^{(k)} }.
  \label{eq:proposal_thetat}
\end{align}
The following theorem proves that $B'$ is a valid proposal. That is,
$B'$ \emph{bounds} $B$ in the sense discussed in
Section~\ref{ssec:prob_trans}, and therefore given $B'$ we can sample $B$.

\begin{theorem}[Validity of Proposal] For any $c$ and $c'$ such that
  $0 \leq c, c' \leq n-1$, we have 
  \begin{align}
    \Lambda_{cc'} \leq \Lambda'_{cc'}.
  \end{align}
  \label{thm:validity}
\end{theorem}
\begin{proof}
  See Appendix \ref{sec:ProofTheorrefthm:validity}. 
\end{proof}
Also see Algorithm~\ref{alg:bdp_sample} of
Appendix~\ref{sec:pseudo_code} for the pseudo-code of the overall
algorithm.

\subsection{Time Complexity}
\label{sec:TimeComplexity}

As it takes $\Theta\rbr{d}$ time to generate each edge in BDP, let us
calculate the expected number of edges $B'$ will generate. The
following quantities similar to \eqref{eq:skg_edgenum} and
\eqref{eq:mag_edgenum} will be found useful
\begin{align}
  e_{MK} &= n \cdot \prod_{k=1}^d \rbr{\sum_{0 \leq a,b \leq 1}
    \mu^{a}\rbr{1-\mu}^{1-a} \theta_{ab}^{(k)}},\\
  e_{KM} &= n \cdot \prod_{k=1}^d \rbr{\sum_{0 \leq a,b \leq 1}
    \mu^{b}\rbr{1-\mu}^{1-b} \theta_{ab}^{(k)}}.
\end{align}
In general, $e_{MK}$ and $e_{KM}$ are not necessarily lower or upper
bounded by $e_{M}$ or $e_{K}$. However, for many of known parameter
values for KPGM and MAGM, especially those considered in
\citet{KimLes10} and \citet{YunVis12}, we empirically observe that they
are indeed between $e_{M}$ and $e_{K}$
\begin{align}
  \min \cbr{e_M, e_K} \leq e_{MK}, e_{KM} \leq \max
  \cbr{e_M, e_K}.
  \label{eq:empirical_obs}
\end{align}
see Figure~\ref{fig:e_mk_mm} for a graphical illustration.

\begin{figure}
  \begin{tikzpicture}[scale=0.45]
    \begin{axis}[
      legend pos=south east,
      width=250pt, height=200pt, no markers,
      xlabel={$\mu$},
      ylabel={expected number of edges},
      domain=0:1,
      samples=100
      ]

      \addplot+[color=black, no markers, thick]
      {(0.15 * (1-x) * (1-x) + 2 * 0.7 * (1-x) * x + 0.85 * x * x) * 4};

      \addplot+[color=red, no markers, thick, dashed]
      {(0.15 + 2 * 0.7 + 0.85)/4 * 4};

      \addplot+[color=blue, no markers, thick, loosely dashed]
      {(0.15 * (1-x) + 0.7 * (1-x) + 0.7 * x + 0.85 * x)/2 * 4};

      \legend{$e_{M}$, $e_{K}$, $e_{KM}\text{,} e_{MK}$}
    \end{axis}
  \end{tikzpicture}
  \begin{tikzpicture}[scale=0.45]
    \begin{axis}[
      legend pos=south east,
      width=250pt, height=200pt, no markers,
      xlabel={$\mu$},
      ylabel={expected number of edges},
      domain=0:1,
      samples=100
      ]

      \addplot+[color=black, no markers, thick]
      {(0.35 * (1-x) * (1-x) + 2 * 0.52 * (1-x) * x + 0.95 * x * x) * 4};

      \addplot+[color=red, no markers, thick, dashed]
      {(0.35 + 2 * 0.52 + 0.95)/4 * 4};

      \addplot+[color=blue, no markers, thick, loosely dashed]
      {(0.35 * (1-x) + 0.52 * (1-x) + 0.52 * x + 0.95 * x)/2 * 4};

      \legend{$e_{M}$, $e_{K}$, $e_{KM}\text{,} e_{MK}$}
    \end{axis}
  \end{tikzpicture}

  \caption{Values of $e_{M}$, $e_{K}$, $e_{KM}$ and
    $e_{MK}$ when $d=1$ and $\Theta = (0.15, 0.7; 0.7, 0.85)$ or
    $\Theta = (0.35, 0.52; 0.52, 0.95)$ was used. One can see that
    $e_{KM}$ and $e_{MK}$ are between $e_{M}$ and
    $e_{K}$, but for general $\Theta$ it may not be the case.}
  \label{fig:e_mk_mm}
\end{figure}
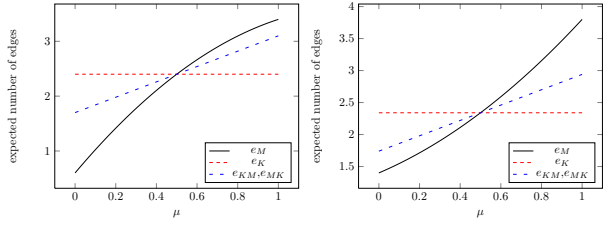

From straightforward calculation, $B^{(\Fcal\Fcal)}$,
$B^{(\Fcal\Ical)}$, $B^{(\Ical\Fcal)}$ and $B^{(\Ical\Ical)}$ generates
$m_{\Fcal}^2 e_{M}$, $m_{\Fcal} m_{\Ical} e_{MK}$, $m_{\Ical}m_{\Fcal}
e_{KM}$ and $m_{\Ical}^2 e_{K}$ edges in expectation, respectively. By
Theorem~\ref{thm:color_bound}, the overall time complexity is $O\rbr{d
  \cdot \rbr{\log_2 n}^2 \cdot \rbr{e_K + e_{KM} + e_{MK} + e_{M}}}$
with high probability. When \eqref{eq:empirical_obs} holds, it can be
further simplified to $O\rbr{d \cdot \rbr{\log_2 n}^2 \cdot \rbr{e_K +
    e_{M}}}$. Note that $d$ is also usually chosen to be $d \leq \log_2
n$. This implies that the time complexity of the whole algorithm is
almost linear in the number of expected edges in MAGM and an equivalent
KPGM.

Note that the time complexity of algorithm in \citet{YunVis12} is at
least $\Omega(d \cdot e_K)$ and attains the best guarantee of $O(d
\rbr{\log_2 n}^2 e_K)$ when $e_M = e_K$. When
\eqref{eq:empirical_obs} holds, therefore, our algorithm is at least
as efficient as their algorithm.

\subsection{Combining two Algorithms}
\label{sec:combine_algs}

Note that one can combine our algorithm and the algorithm of
\citet{YunVis12} to get improved performance. For both algorithms, it
only takes $O\rbr{nd}$ time to estimate the expected running time. Thus
one can always select the best algorithm for a given set of parameter
values.

\section{Experiments}

We empirically evaluated the efficiency and scalability of our sampling
algorithm. Our experiments are designed to answer the following
questions: 1) How does our algorithm scale as a function of $e_M$, the
expected number of edges in the graph? 2) What is the advantage of using
our algorithm compared to that of \citet{YunVis12}?

Our algorithm is implemented in \texttt{C++} and will be made available
for download from \url{http://anonymous}. For quilting algorithm, we
used the original implementation of \citet{YunVis12} which is also
written in \texttt{C++} and compiled with the same options. All
experiments are run on a machine with a 2.1 GHz processor running Linux.

Following \citet{YunVis12}, we uniformly set $n=2^d$, and used the
same $\Theta$ matrices and $\mu$ values at all levels: i.e., $\Theta =
\Theta^{(1)} = \Theta^{(2)} = \cdots = \Theta^{(d)}$ and $\mu =
\mu^{(1)} = \cdots = \mu^{(d)}$. Furthermore, we experimented with the
following $\Theta$ matrices used by \citet{KimLes10} and
\citet{MorNev09} to model real world graphs:
\begin{align}
  \Theta_1 =
  \mymatrix{cc}{
    0.15 & 0.7 \\
    0.7 & 0.85
  } \text{ and }
  \Theta_2 =
  \mymatrix{cc}{
    0.35 & 0.52 \\
    0.52 & 0.95
  }
  .
  \nonumber
\end{align}

\begin{figure}
  {\small (a) $\mu = 0.2$}
  \begin{center}
    \begin{tikzpicture}[scale=0.45]
      \begin{axis}[
        title={$\Theta_1$, $\mu=0.2$}, xlabel={Number of Expected Edges $e_M$},
        ylabel={Running Time (s)},
        cycle list name=black white, legend pos=north west]

        \addplot+[
        color=blue, mark=x, thick,
        error bars/.cd,
        y dir=both, y explicit,
        error mark=-]
        table[x=x,y=y,y error=errory]
        {icml_mus_0.2_1.dat};

        \addplot+[
        color=red, mark=o, mark options={solid}, thick, loosely dashed,
        error bars/.cd,
        y dir=both, y explicit,
        error mark=-]
        table[x=x,y=y,y error=errory]
        {cluster_mus_0.2_1.dat};

        \legend{BDP Sampler, Quilting}

      \end{axis}
    \end{tikzpicture}
    \begin{tikzpicture}[scale=0.45]
      \begin{axis}[
        title={$\Theta_2$, $\mu = 0.2$}, xlabel={Number of Expected Edges $e_M$},
        ylabel={Running Time (s)},
        cycle list name=black white, legend pos=north west]

        \addplot+[
        color=blue, mark=x, thick,
        error bars/.cd,
        y dir=both, y explicit,
        error mark=-]
        table[x=x,y=y,y error=errory]
        {icml_mus_0.2_2.dat};

        \addplot+[
        color=red, mark=o, mark options={solid}, thick, loosely dashed,
        error bars/.cd,
        y dir=both, y explicit,
        error mark=-]
        table[x=x,y=y,y error=errory]
        {cluster_mus_0.2_2.dat};

        \legend{BDP Sampler, Quilting}

      \end{axis}
    \end{tikzpicture}
  \end{center}

  {\small (b) $\mu = 0.3$}
  \begin{center}
    \begin{tikzpicture}[scale=0.45]
      \begin{axis}[
        title={$\Theta_1$, $\mu=0.3$}, xlabel={Number of Expected Edges $e_M$},
        ylabel={Running Time (s)},
        cycle list name=black white, legend pos=north west]

        \addplot+[
        color=blue, mark=x, thick,
        error bars/.cd,
        y dir=both, y explicit,
        error mark=-]
        table[x=x,y=y,y error=errory]
        {icml_mus_0.3_1.dat};

        \addplot+[
        color=red, mark=o, mark options={solid}, thick, loosely dashed,
        error bars/.cd,
        y dir=both, y explicit,
        error mark=-]
        table[x=x,y=y,y error=errory]
        {cluster_mus_0.3_1.dat};

        \legend{BDP Sampler, Quilting}

      \end{axis}
    \end{tikzpicture}
    \begin{tikzpicture}[scale=0.45]
      \begin{axis}[
        title={$\Theta_2$, $\mu=0.3$}, xlabel={Number of Expected Edges $e_M$},
        ylabel={Running Time (s)},
        cycle list name=black white, legend pos=north west]

        \addplot+[
        color=blue, mark=x, thick,
        error bars/.cd,
        y dir=both, y explicit,
        error mark=-]
        table[x=x,y=y,y error=errory]
        {icml_mus_0.3_2.dat};

        \addplot+[
        color=red, mark=o, mark options={solid}, thick, loosely dashed,
        error bars/.cd,
        y dir=both, y explicit,
        error mark=-]
        table[x=x,y=y,y error=errory]
        {cluster_mus_0.3_2.dat};

        \legend{BDP Sampler, Quilting}

      \end{axis}
    \end{tikzpicture}
  \end{center}
  {\small (c) $\mu = 0.5$}
  \begin{center}
    \begin{tikzpicture}[scale=0.45]
      \begin{axis}[
        title={$\Theta_1$, $\mu=0.5$}, xlabel={Number of Expected Edges $e_M$},
        ylabel={Running Time (s)},
        cycle list name=black white, legend pos=north west]

        \addplot+[
        color=blue, mark=x, thick,
        error bars/.cd,
        y dir=both, y explicit,
        error mark=-]
        table[x=x,y=y,y error=errory]
        {icml_mus_0.5_1.dat};

        \addplot+[
        color=red, mark=o, mark options={solid}, thick, loosely dashed,
        error bars/.cd,
        y dir=both, y explicit,
        error mark=-]
        table[x=x,y=y,y error=errory]
        {cluster_mus_0.5_1.dat};

        \legend{BDP Sampler, Quilting}

      \end{axis}
    \end{tikzpicture}
    \begin{tikzpicture}[scale=0.45]
      \begin{axis}[
        title={$\Theta_2$, $\mu = 0.5$}, xlabel={Number of Expected Edges $e_M$},
        ylabel={Running Time (s)},
        cycle list name=black white, legend pos=north west]

        \addplot+[
        color=blue, mark=x, thick,
        error bars/.cd,
        y dir=both, y explicit,
        error mark=-]
        table[x=x,y=y,y error=errory]
        {icml_mus_0.5_2.dat};

        \addplot+[
        color=red, mark=o, mark options={solid}, thick, loosely dashed,
        error bars/.cd,
        y dir=both, y explicit,
        error mark=-]
        table[x=x,y=y,y error=errory]
        {cluster_mus_0.5_2.dat};

        \legend{BDP Sampler, Quilting}

      \end{axis}
    \end{tikzpicture}
  \end{center}
  {\small (d) $\mu = 0.7$}
  \begin{center}
    \begin{tikzpicture}[scale=0.45]
      \begin{axis}[
        title={$\Theta_1$, $\mu=0.7$}, xlabel={Number of Expected Edges $e_M$},
        ylabel={Running Time (s)},
        cycle list name=black white, legend pos=north west]

        \addplot+[
        color=blue, mark=x, thick,
        error bars/.cd,
        y dir=both, y explicit,
        error mark=-]
        table[x=x,y=y,y error=errory]
        {icml_mus_0.7_1.dat};

        \addplot+[
        color=red, mark=o, mark options={solid}, thick, loosely dashed,
        error bars/.cd,
        y dir=both, y explicit,
        error mark=-]
        table[x=x,y=y,y error=errory]
        {cluster_mus_0.7_1.dat};

        \legend{BDP Sampler, Quilting}

      \end{axis}
    \end{tikzpicture}
    \begin{tikzpicture}[scale=0.45]
      \begin{axis}[
        title={$\Theta_2$, $\mu = 0.7$}, xlabel={Number of Expected Edges $e_M$},
        ylabel={Running Time (s)},
        cycle list name=black white, legend pos=north west]

        \addplot+[
        color=blue, mark=x, thick,
        error bars/.cd,
        y dir=both, y explicit,
        error mark=-]
        table[x=x,y=y,y error=errory]
        {icml_mus_0.7_2.dat};

        \addplot+[
        color=red, mark=o, mark options={solid}, thick, loosely dashed,
        error bars/.cd,
        y dir=both, y explicit,
        error mark=-]
        table[x=x,y=y,y error=errory]
        {cluster_mus_0.7_2.dat};

        \legend{BDP Sampler, Quilting}

      \end{axis}
    \end{tikzpicture}
  \end{center}
  {\small (e) $\mu = 0.9$}
  \begin{center}
    \begin{tikzpicture}[scale=0.45]
      \begin{axis}[
        title={$\Theta_1$, $\mu=0.9$}, xlabel={Number of Expected Edges $e_M$},
        ylabel={Running Time (s)},
        cycle list name=black white, legend pos=north west]

        \addplot+[
        color=blue, mark=x, thick,
        error bars/.cd,
        y dir=both, y explicit,
        error mark=-]
        table[x=x,y=y,y error=errory]
        {icml_mus_0.9_1.dat};

        \addplot+[
        color=red, mark=o, mark options={solid}, thick, loosely dashed,
        error bars/.cd,
        y dir=both, y explicit,
        error mark=-]
        table[x=x,y=y,y error=errory]
        {cluster_mus_0.9_1.dat};

        \legend{BDP Sampler, Quilting}

      \end{axis}
    \end{tikzpicture}
    \begin{tikzpicture}[scale=0.45]
      \begin{axis}[
        title={$\Theta_2$, $\mu = 0.9$}, xlabel={Number of Expected Edges $e_M$},
        ylabel={Running Time (s)},
        cycle list name=black white, legend pos=north west]

        \addplot+[
        color=blue, mark=x, thick,
        error bars/.cd,
        y dir=both, y explicit,
        error mark=-]
        table[x=x,y=y,y error=errory]
        {icml_mus_0.9_2.dat};

        \addplot+[
        color=red, mark=o, mark options={solid}, thick, loosely dashed,
        error bars/.cd,
        y dir=both, y explicit,
        error mark=-]
        table[x=x,y=y,y error=errory]
        {cluster_mus_0.9_2.dat};

        \legend{BDP Sampler, Quilting}

      \end{axis}
    \end{tikzpicture}
  \end{center}
  \caption{Comparison of running time (in seconds) of our algorithm
    vs the quilting algorithm of \citet{YunVis12} as a function of
    expected number of edges $e_M$ for two different values of $\Theta$
    and five values of $\mu$.}
  \label{fig:func_of_e}
\end{figure}
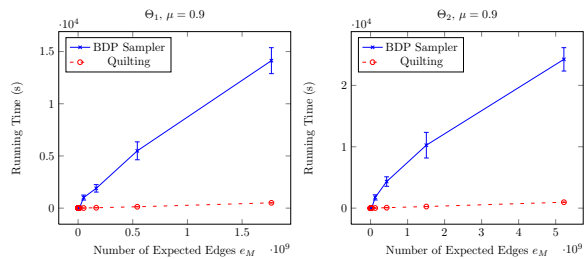

Figure~\ref{fig:func_of_e} shows the running time of our algorithm vs
\citet{YunVis12} as a function of expected number of edges $e_M$.  Each
experiment was repeated ten times to obtain error bars. As our algorithm
has theoretical time complexity guarantee, irrespective of $\mu$ the
running time is almost linear in $e_M$. On the other hand,
\citet{YunVis12} shows superb performance when dealing with relatively
dense graphs ($\mu > 0.5$), but when dealing with sparser graphs ($\mu <
0.5$) our algorithm outperforms.

Figure~\ref{fig:func_of_mu} shows the dependence of running time on
$\mu$ more clearly. In our parameter setting, the number of expected
edges is an increasing function of $\mu$ (see Figure~\ref{fig:e_mk_mm}
for $d=1$). As the time complexity of our algorithm depends on $e_M$,
the running time of our algorithm increases accordingly as $\mu$
increases. In the case of quilting algorithm, however, the running time
is almost symmetric with respect to $\mu = 0.5$. Thus, when $\mu < 0.5$
it is relatively inefficient, compared to when $\mu \geq 0.5$.

\begin{figure}
  {\small (a) $\mu \leq 0.5$ }
  \begin{center}
    \begin{tikzpicture}[scale=0.45]
      \begin{axis}[
        title={$\Theta_1$, $n=2^{17}$}, xlabel={$\mu$},
        ylabel={Running Time (s)},
        cycle list name=black white, legend pos=north west]

        \addplot+[
        color=blue, mark=x, thick,
        error bars/.cd,
        y dir=both, y explicit,
        error mark=-]
        table[x=x,y=y,y error=errory]
        {icml_mus_t_d_17_1.dat};

        \addplot+[
        color=red, mark=o, mark options={solid}, thick, loosely dashed,
        error bars/.cd,
        y dir=both, y explicit,
        error mark=-]
        table[x=x,y=y,y error=errory]
        {cluster_mus_t_d_17_1.dat};

        \legend{BDP Sampler, Quilting}

      \end{axis}
    \end{tikzpicture}
    \begin{tikzpicture}[scale=0.45]
      \begin{axis}[
        title={$\Theta_2$, $n=2^{17}$}, xlabel={$\mu$},
        ylabel={Running Time (s)},
        cycle list name=black white, legend pos=north west]

        \addplot+[
        color=blue, mark=x, thick,
        error bars/.cd,
        y dir=both, y explicit,
        error mark=-]
        table[x=x,y=y,y error=errory]
        {icml_mus_t_d_17_2.dat};

        \addplot+[
        color=red, mark=o, mark options={solid}, thick, loosely dashed,
        error bars/.cd,
        y dir=both, y explicit,
        error mark=-]
        table[x=x,y=y,y error=errory]
        {cluster_mus_t_d_17_2.dat};

        \legend{BDP Sampler, Quilting}

      \end{axis}
    \end{tikzpicture}
  \end{center}

  {\small (b) General Value of $\mu$}
  \begin{center}
    \begin{tikzpicture}[scale=0.45]
      \begin{axis}[
        title={$\Theta_1$, $n=2^{17}$}, xlabel={$\mu$},
        ylabel={Running Time (s)},
        cycle list name=black white, legend pos=north west]

        \addplot+[
        color=blue, mark=x, thick,
        error bars/.cd,
        y dir=both, y explicit,
        error mark=-]
        table[x=x,y=y,y error=errory]
        {icml_mus_d_17_1.dat};

        \addplot+[
        color=red, mark=o, mark options={solid}, thick, loosely dashed,
        error bars/.cd,
        y dir=both, y explicit,
        error mark=-]
        table[x=x,y=y,y error=errory]
        {cluster_mus_d_17_1.dat};

        \legend{BDP Sampler, Quilting}

      \end{axis}
    \end{tikzpicture}
    \begin{tikzpicture}[scale=0.45]
      \begin{axis}[
        title={$\Theta_2$, $n=2^{17}$}, xlabel={$\mu$},
        ylabel={Running Time (s)},
        cycle list name=black white, legend pos=north west]

        \addplot+[
        color=blue, mark=x, thick,
        error bars/.cd,
        y dir=both, y explicit,
        error mark=-]
        table[x=x,y=y,y error=errory]
        {icml_mus_d_17_2.dat};

        \addplot+[
        color=red, mark=o, mark options={solid}, thick, loosely dashed,
        error bars/.cd,
        y dir=both, y explicit,
        error mark=-]
        table[x=x,y=y,y error=errory]
        {cluster_mus_d_17_2.dat};

        \legend{BDP Sampler, Quilting}

      \end{axis}
    \end{tikzpicture}
  \end{center}
  \caption{Comparison of running time (in seconds) of our
    algorithm vs the quilting algorithm of \citet{YunVis12} as a
    function of $\mu$ for two different values of $\Theta$ and
    $n=2^{17}$. }
  \label{fig:func_of_mu}
\end{figure}

\section{Conclusion}
\label{sec:Conclusion}

We introduced a novel and efficient sampling algorithm for the MAGM. The
run-time of our algorithm depends on $e_{K}$ and $e_M$. For sparse
graphs, which are primarily of interest in applications, the value of
$e_{M}$ is well bounded, and our method is able to outperform the
quilting algorithm. However, when $\mu$ is greater than $0.5$, MAGM
produces dense graphs. In this case the heuristics of \citet{YunVis12}
work well in practice. One can combine the two algorithms to produce a
fast hybrid algorithm. Theoretical investigation of the quilting
algorithm and its heuristics may provide more insights into improving
both algorithms.

For the parameter settings we studied the corresponding KPGM graphs are
sparse and can be sampled efficiently. However, for some values of
$\Theta$ the corresponding KPGM graphs can become dense and difficult to
sample. Removing dependency of time complexity on $e_K$ remains an open
question, and a focus of our future research.

\bibliography{accrej.bib}
\bibliographystyle{icml2012}

\newpage

\appendix

\section{Technical Proofs (not included in 8 page limit)}
\label{sec:proofs}

\subsection{Proof of Theorem~\ref{thm:bdp}}
\label{sec:ProofTheorrefthm:bdp}

\begin{proof}
  By conditioning on the number of edges $\abr{\Ecal}$, the probability
  mass function can be written as
  \begin{align}
    \PP\sbr{A} &= \PP \sbr{ \abr{\Ecal}} \cdot \PP \sbr{A \mid \abr{\Ecal} }.
  \end{align}
  Recall that the marginal distribution of $\abr{\Ecal}$ follows Poisson
  distribution with rate parameter $e_K$.  Using
  \eqref{eq:skg_edgenum} and the definition of a Poisson probability
  mass function,
  \begin{align}
    P\sbr{\abr{\Ecal} } = \exp\rbr{-\sum_{i,j=1}^n \Gamma_{ij}} \frac{\rbr{\sum_{i,j=1}^n
        \Gamma_{ij}}^{\abr{\Ecal}}}{\abr{\Ecal}!}.
  \end{align}
  On the other hand, the conditional distribution of $A$ given $\abr{\Ecal}$
  is defined by the multinomial distribution, and its probability mass
  function is given by
  \begin{align}
    \PP\sbr{A \mid \abr{\Ecal}} &= \binom{\abr{\Ecal}}{A_{1,1} A_{1,2} \cdots
      A_{n, n}} \nonumber \\ & \;\;\;\;\; \cdot \prod_{i,j=1}^n
    \rbr{\frac{\Gamma_{ij}}{\sum_{i,j=1}^n \Gamma_{ij}}}^{A_{ij}},
  \end{align}
  where $\binom{\abr{\Ecal}}{A_{1,1} A_{1,2} \cdots A_{n, n}}$ is the
  multinomial coefficient. By definition $|E| := \sum_{i,j=1}^n
  A_{ij}$ and after some simple algebra, we have
  \begin{align}
    \PP\sbr{A} = \prod_{i,j=1}^n \exp\rbr{-\Gamma_{ij}}
    \frac{\Gamma_{ij}^{A_{ij}}} {A_{ij}!}.
  \end{align}
  By the factorization theorem, every $A_{ij}$ is independent of each
  other. Furthermore, $A_{ij}$ follows a Poisson distribution with rate
  parameter $\Gamma_{ij}$.
\end{proof}

\subsection{Proof of Theorem~\ref{thm:color_bound}}
\begin{proof}
  For $c \in \Fcal$, we apply the multiplicative form of
  Hoeffding-Chernoff inequality (Chapter 35.1, \citet{Dasgupta2011})
  to get
  \begin{align}
    \PP\sbr{\abr{\Vcal_c} \geq \log_2 n \cdot \EE\sbr{\abr{\Vcal_c}}}
    &< \rbr{\frac{\exp\rbr{\log_2 n - 1}}{\rbr{\log_2 n}^{\log_2
          n}}}^{\EE\sbr{\abr{\Vcal_c}}} \\
    &\leq \rbr{\frac{\exp\rbr{\log_2 n - 1}}{\rbr{\log_2 n}^{\log_2
          n}}},
  \end{align}
  for large enough $n$. Then using the union bound
  \begin{align}
    &\PP\sbr{\bigcup_{c \in \Fcal} \abr{\Vcal_c} \geq \log_2 n \cdot
      \EE\sbr{\abr{\Vcal_c}}} \leq \sum_{c \in \Fcal} \PP\sbr{\abr{\Vcal_c} \geq
      \log_2 n \cdot
      \EE\sbr{\abr{\Vcal_c}}} \\
    &\;\;\;\leq n \cdot \rbr{\frac{\exp\rbr{\log_2 n - 1}}{\rbr{\log_2
          n}^{\log_2 n}}} \rightarrow 0
  \end{align}
  as $n \rightarrow \infty$.  For $c \in \Ical$, on the other hand, we
  apply the additive form of Hoeffding-Chernoff inequality:
  \begin{align}
    &\PP\sbr{\abr{\Vcal_c} \geq \log_2 n} 
    < 
    \rbr{\frac{\EE\sbr{\abr{\Vcal_c}}}{\log_2 n}}^{\log_2 n} 
    \cdot
    \rbr{\frac{1-\EE\sbr{\abr{\Vcal_c}}/n}{1-\log_2 n/n}}^{n - \log_2 n}\\
    &\;\;\leq n \rbr{\frac 1 {\log_2 n} \cdot \frac{1-\log_2
        n/n}{1-1/n}}^{\log_2 n}.
  \end{align}
  Using union bound again, 
  \begin{align}
    \PP\sbr{\bigcup_{c \in \Ical} \abr{\Vcal_c} \geq \log_2 n }
    \rightarrow 0.
  \end{align}
\end{proof}

\subsection{Proof of Theorem~\ref{thm:validity}}
\label{sec:ProofTheorrefthm:validity}
\begin{proof}
  Let $\Lambda^{(\Acal, \Bcal)}$ be rate parameter matrix of
  $B^{(\Acal\Bcal)}$. From the definition \eqref{eq:kpgm_p2}, 
  \begin{align}
    \Lambda^{(\Acal, \Bcal)} := \Theta^{(\Acal,\Bcal)(1)} \otimes
    \Theta^{(\Acal,\Bcal)(2)} \otimes \cdots \otimes \Theta^{(\Acal,
      \Bcal)(d)}.
    \label{eq:proposal_r}
  \end{align}
  From \eqref{eq:proposal_thetat} and
  \eqref{eq:proposal_r}, it is easy to verify that
  \begin{align}
    \begin{array}{lcl}
      \Lambda^{(\Fcal\Fcal)}_{cc'} = \rbr{m_{\Fcal}}^2 \cdot 
      \EE\sbr{\abr{\Vcal_c}} \cdot \EE\sbr{\abr{\Vcal_{c'}}} \cdot \Gamma_{cc'} & \text{
        for } & c \in \Fcal, c' \in \Fcal,\\
      \Lambda^{(\Fcal\Ical)}_{cc'} = m_{\Fcal} \cdot m_{\Ical} \cdot \EE\sbr{\abr{\Vcal_c}} \cdot \Gamma_{cc'} & \text{
        for } & c \in \Fcal, c' \in \Ical,\\
      \Lambda^{(\Ical\Fcal)}_{cc'} = m_{\Ical} \cdot m_{\Fcal} \cdot \EE\sbr{\abr{\Vcal_{c'}}} \cdot \Gamma_{cc'} & \text{
        for } & c \in \Ical, c' \in \Fcal,\\
      \Lambda^{(\Ical\Ical)}_{cc'} = \rbr{m_{\Ical}}^2 \Gamma_{cc'} & \text{
        for } &c \in \Ical, c' \in \Ical. \nonumber
    \end{array}
  \end{align}
  Using \eqref{eq:r_cc_def} and \eqref{eq:def_m_f}
  obtains
  \begin{align}
    \Lambda_{cc'} \leq \Lambda^{(\Acal\Bcal)}_{cc'} \leq \Lambda'_{cc'},
  \end{align}
  for any $\Acal, \Bcal \in \cbr{\Fcal, \Ical}$, $c \in \Acal$, and $c'
  \in \Bcal$.
\end{proof}

\newpage

\section{Pseudo-Code of Algorithms}
\label{sec:pseudo_code}

\begin{algorithm}[ht]
  \caption{Description of Ball-Dropping Process}
  \label{alg:kpgmsample}
  \begin{algorithmic}
    \STATE {\bfseries Function \texttt{BDP}}
    \STATE {\bfseries Input:} parameter $\Thetat$
    \STATE {\bfseries Output:} set of edges $\Ecal$
    
    \STATE $\Ecal \gets \emptyset$
    \STATE $e_K \gets \prod_{k=1}^d\rbr{\theta^{(k)}_{00} +\theta^{(k)}_{01}
      +\theta^{(k)}_{10}+\theta^{(k)}_{11}}$
    \STATE Generate $X \sim Poisson(e_K)$.
    \FOR{$x=1$ to $X$}
    \STATE $S_{start}, T_{start} \gets 1$
    \STATE $S_{end}, T_{end} \gets n$
    \FOR{$k \gets 1$ to $d$}
    \STATE Sample $(a,b) \propto \theta_{ab}^{(k)}$
    \STATE $S_{start} \gets S_{start} +  \frac {an} {2^k}$.
    \STATE $T_{start} \gets T_{start} + \frac {bn} {2^k}$.
    \STATE $S_{end} \gets S_{end} -  \frac {(1-a)n} {2^k}$.
    \STATE $T_{end} \gets T_{end} - \frac {(1-b)n} {2^k}$.
    \ENDFOR
    \STATE {\# We have $S_{start} = S_{end}$,
      $T_{start} = T_{end}$}
    \STATE $\Ecal \gets \Ecal \cup \{ (S_{start}, T_{start}) \}$
    \ENDFOR
  \end{algorithmic}
\end{algorithm}

\begin{algorithm}[t]
  \caption{BDP Sampler of MAGM}
  \label{alg:bdp_sample}
  \begin{algorithmic}
    \STATE {\bfseries Input:} parameters $\Thetat$, $\mut$
    \STATE {\bfseries Output:} set of edges $\Ecal$
    \STATE $\Ecal \gets \emptyset$
    \FOR{$\Acal$ in $\cbr{\Fcal, \Ical}$}
    \FOR{$\Bcal$ in $\cbr{\Fcal, \Ical}$}
    \FOR{$(c,c')$ in $\texttt{BDP}\rbr{\Thetat^{(\Acal,\Bcal)}}$}
    \IF{$c \in \Acal$ and $c' \in \Bcal$}
    \STATE Generate $u \sim Uniform(0,1)$.
    \IF{$u \leq \frac{\Lambda_{ij}}{\Lambda_{ij}^{(\Acal\Bcal)}}$}
    \STATE Sample $i$ uniformly from $\Vcal_c$.
    \STATE Sample $j$ uniformly from $\Vcal_{c'}$.
    \STATE $\Ecal \gets \Ecal \cup \cbr{(i,j)}$.
    \ENDIF
    \ENDIF
    \ENDFOR
    \ENDFOR
    \ENDFOR
  \end{algorithmic}
\end{algorithm}

\end{document}